\documentclass[journal]{IEEEtran}
\usepackage{amsmath,amsfonts}
\usepackage{algorithmic}
\usepackage{algorithm}
\usepackage{array}
\usepackage[caption=false,font=normalsize,labelfont=sf,textfont=sf]{subfig}
\usepackage{textcomp}
\usepackage{stfloats}
\usepackage{url}
\usepackage{verbatim}
\usepackage{graphicx}
\usepackage{cite}
\usepackage{array}
\usepackage{mathtools}
\usepackage{amssymb}
\usepackage{amsthm}
\usepackage{booktabs}
\usepackage{multirow}
\usepackage{xcolor}

\theoremstyle{plain}
\newtheorem{theorem}{Theorem}
\newtheorem{lemma}{Lemma}
\newtheorem{assumption}{Assumption}

\theoremstyle{remark}

\begin{document}

\title{Real-Time Projected Adaptive Control for Closed-Chain Co-Manipulative Continuum Robots}


\author{Rana Danesh,
        Farrokh~Janabi-Sharifi,
        and Farhad~Aghili,%
\thanks{Rana Danesh and Farrokh Janabi-Sharifi are with the Department of Mechanical, Industrial, and Mechatronics Engineering, Toronto Metropolitan University, Toronto, ON M5B 2K3, Canada (e-mail: rana.danesh@torontomu.ca; fsharifi@torontomu.ca).}%
\thanks{Farhad Aghili is with the Department of Mechanical, Industrial, and Aerospace Engineering, Concordia University, Montreal, QC H3G 1M8, Canada (e-mail: farhad.aghili@concordia.ca).}}

\maketitle

\begin{abstract}
In co-manipulative continuum robots (CCRs), multiple continuum arms cooperate by grasping a common flexible object, forming a closed-chain deformable mechanical system. The closed-chain coupling induces strong dynamic interactions and internal reaction forces. Moreover, in practical tasks, the flexible object's physical parameters are often unknown and vary between operations, rendering nominal model-based controllers inadequate.
This paper presents a projected adaptive control framework for CCRs formulated at the dynamic level. The coupled dynamics are expressed using the Geometric Variable Strain (GVS) representation, yielding a finite-dimensional model that accurately represents the system, preserves the linear-in-parameters structure required for adaptive control, and is suitable for real-time implementation. Closed-chain interactions are enforced through Pfaffian velocity constraints, and an orthogonal projection is used to express the dynamics in the constraint-consistent motion subspace. Based on the projected dynamics, an adaptive control law is developed to compensate online for uncertain dynamic parameters of both the continuum robots and the manipulated flexible object. Lyapunov analysis establishes closed-loop stability and convergence of the task-space tracking errors to zero. Simulation and experiments on a tendon-driven CCR platform validate the proposed framework in task-space regulation and trajectory tracking.
\end{abstract}


\vspace{-1 em}
\section{Introduction}
\IEEEPARstart{C}{ontinuum} robots (CRs) have emerged as a powerful alternative to rigid manipulators in tasks requiring safe interaction, adaptability to confined environments, and compliant manipulation \cite{webster2010design,cianchetti2014soft,burgner2015continuum,laschi2016lessons}. 
Their intrinsic flexibility and distributed actuation enable operations that are difficult or infeasible for conventional robots, particularly in surgical intervention, inspection, and manipulation of delicate or deformable objects \cite{simaan2018medical,rus2015design,trivedi2008soft}. However, many practical tasks exceed the capability of a single CR and instead require co-manipulation, in which multiple continuum arms interact mechanically with a shared flexible object \cite{jalali2021dynamic,chikhaoui2018toward}. Such systems can provide enhanced load capacity, expanded workspace coverage, and increased task versatility, while preserving the compliance advantages of individual CRs \cite{lotfavar2017cooperative}. 
Despite these advantages, co-manipulation introduces strong mechanical coupling that transforms the overall system into a closed-chain deformable mechanism whose allowable motion is restricted by loop-closure constraints. This coupling fundamentally alters the system dynamics and makes both modeling and control considerably more challenging \cite{li2025collaborative}.

Although CCR modeling has progressed from kinematic and kinetostatic descriptions to dynamic and coupled continuum formulations, existing control strategies remain Jacobian-based and focused on feedback-level coordination of the individual arms rather than control of the CCR as a unified closed-chain dynamic system \cite{chikhaoui2018toward,peng2023modeling,norouzi2021switching}. In co-manipulation, the CRs and the shared flexible object are mechanically coupled and therefore cannot be treated independently. Their motion must remain physically compatible under the closed-chain constraints, which gives rise to internal reaction forces and moments that cannot be neglected in the system dynamics. Because the shared object possesses its own deformation states, it affects the dynamics not merely as a payload, but as a deformable subsystem within the closed chain. Moreover, the distributed dynamic properties of both the CRs and the manipulated flexible object may be uncertain, and these parametric uncertainties propagate through the coupled dynamics. This motivates the need for adaptive control to compensate for such uncertainty online within the closed-chain CCR system.

Accordingly, this paper develops an adaptive control framework for CCRs under closed-chain constraints and parametric uncertainty. Using the Geometric Variable Strain (GVS) representation\cite{renda2018discrete}, the coupled dynamics are expressed in a finite-dimensional form that accurately represents the coupled deformable dynamics, preserves the linear-in-parameters structure required for adaptive control, and remains computationally efficient for real-time implementation. Within this framework, an orthogonal projection is used to express the dynamics in the constraint-consistent motion subspace, thereby ensuring that the closed-chain system evolves only along motions that satisfy the imposed constraints. An adaptive control law is then designed to compensate online for uncertain dynamic parameters of both the CRs and the shared flexible object.
The main contributions of this paper are as follows:
\begin{itemize}
  \item Building on the GVS representation, a unified closed-chain dynamic formulation for CCRs is established and projected onto the constraint-consistent motion subspace.

 \item An adaptive control law is developed on the projected closed-chain dynamics to compensate online for uncertain dynamic parameters of both the CRs and the manipulated flexible object. A Lyapunov-based analysis is further provided to establish closed-loop stability and asymptotic convergence of the task-space tracking errors to zero.

 \item The proposed framework is validated in simulation and experiments on a tendon-driven CCR, for both regulation and trajectory tracking against non-adaptive controllers.

\end{itemize}

The remainder of the paper is organized as follows. 
Section~II reviews the related literature. 
Section~III introduces the GVS formulation and derives the closed-chain dynamics. 
Section~IV presents the projected adaptive controller along with its stability analysis. 
Simulation and experimental results are presented in Sections~V and~VI, respectively. 
Section~VII concludes the paper.
\vspace{-0.5 em}
\section{Related Work}

\subsection{Modeling of Co-Manipulative Continuum Robots}
Modeling CCRs is challenging due to the geometric interdependence and physical interactions among individual CRs. Most existing approaches therefore adopt either an independent \cite{ma2022collaborative,cheng2023dexterity,wang2019design} or a coupled \cite{lilge2022kinetostatic,mahoney2016reconfigurable} modeling strategy. Independent modeling treats each CR separately and combines their behaviors at the task level, whereas coupled modeling explicitly enforces interaction constraints among the robots, representing the CCR as a closed-chain system.

Most CCR formulations are built on an underlying backbone model for each CR, typically using either constant-curvature models \cite{nuelle2020modeling,wen2023modeling,wang2024development,dai2023novel,quaicoe2024cooperative} or Cosserat rod theory \cite{wang2021eccentric,mitros2022design}. Constant-curvature models are computationally efficient, but their simplifying assumptions can limit accuracy under contact and external loading \cite{wang2026strain}. 
Cosserat rod theory provides a description of distributed deformation and constitutive behavior, but its governing partial differential equations must be discretized to yield a computationally efficient model. In this work, the GVS approach is adopted to discretize the Cosserat rod by representing the continuous strain field with a finite set of basis functions, yielding explicit dynamic expressions that are suitable for control design and real-time implementation \cite{armanini2021discrete,boyer2020dynamics,mathew2022sorosim}. Moreover, the GVS formulation enables the dynamics of the CCR system to be modeled as a closed-chain system while preserving the linear-in-parameters structure required for adaptive control \cite{renda2020geometric}.

\vspace{-1 em}
\subsection{Control of Co-Manipulative Continuum Robots}
Most CCR control methods are developed for separated collaboration and rely on kinematic, Jacobian-based coordination \cite{yu2023model,tan2021synchronous,zhang2022cooperative,sabetian2019self}, rather than controlling the CCR as a unified closed-chain dynamic system. Within this kinematic framework, many CCR systems implement cooperation in open loop, often relying on structural or design optimization to improve motion accuracy, while master-slave teleoperation remains a common way to command coordinated motions \cite{de2013introducing,hwang2020k,li2020cadaveric,chen2021shurui}. These approaches can be sensitive to operator skill and unmodeled effects, particularly in cluttered environments or during contact \cite{ji2022omnidirectional}.

Closed-loop CCR control relies on real-time shape and pose feedback to compute tracking errors and update the inputs online. Most demonstrations use visual feedback, either from external cameras or onboard sensing between the arms \cite{song2021real,wu2021closed,yang2020closed,xu2009system,reiter2011learning,wang2023vision}. With such feedback, synchronous CCR control is commonly achieved by stacking the individual robot Jacobians into a unified Jacobian and computing coordinated commands for all end effectors within a single control law \cite{wang2025closed}. While this improves online coordination at the kinematic level, it does not explicitly account for the closed-chain dynamics of the coupled deformable robot-object system.

In co-manipulation, the CRs and the shared flexible object are mechanically coupled and their motion must remain physically compatible under the closed-chain constraints. This coupling gives rise to internal reaction forces and moments, while the object's flexibility introduces its own deformation states into the overall dynamics. As a result, the CCR cannot be treated as a collection of independently coordinated arms. Moreover, the distributed dynamic properties of both the CRs and the manipulated flexible object may be uncertain, and these uncertainties propagate through the coupled dynamics. Consequently, what remains missing is an adaptive control framework for CCRs that, to the best of our knowledge, is the first to be formulated directly on the unified closed-chain dynamics and is suitable for online compensation of uncertain parameters.

To address this gap, the following section formulates the CCR dynamics used for adaptive control under closed-chain constraints.

\section{Geometric Variable Strain Modeling}

The schematic configuration of the CCR system is shown in Fig.~\ref{fig:schematic}. The system consists of CRs forming a closed kinematic chain through a flexible object.  
Building on the GVS framework, this section formulates the kinematics and develops the constrained dynamic model of the CCR system \cite{renda2020geometric,renda2018discrete,armanini2021discrete,boyer2020dynamics}.

\begin{figure}
    \centering
    \includegraphics[width=0.4\textwidth,trim=6cm 2.5cm 6.5cm 2cm, clip]{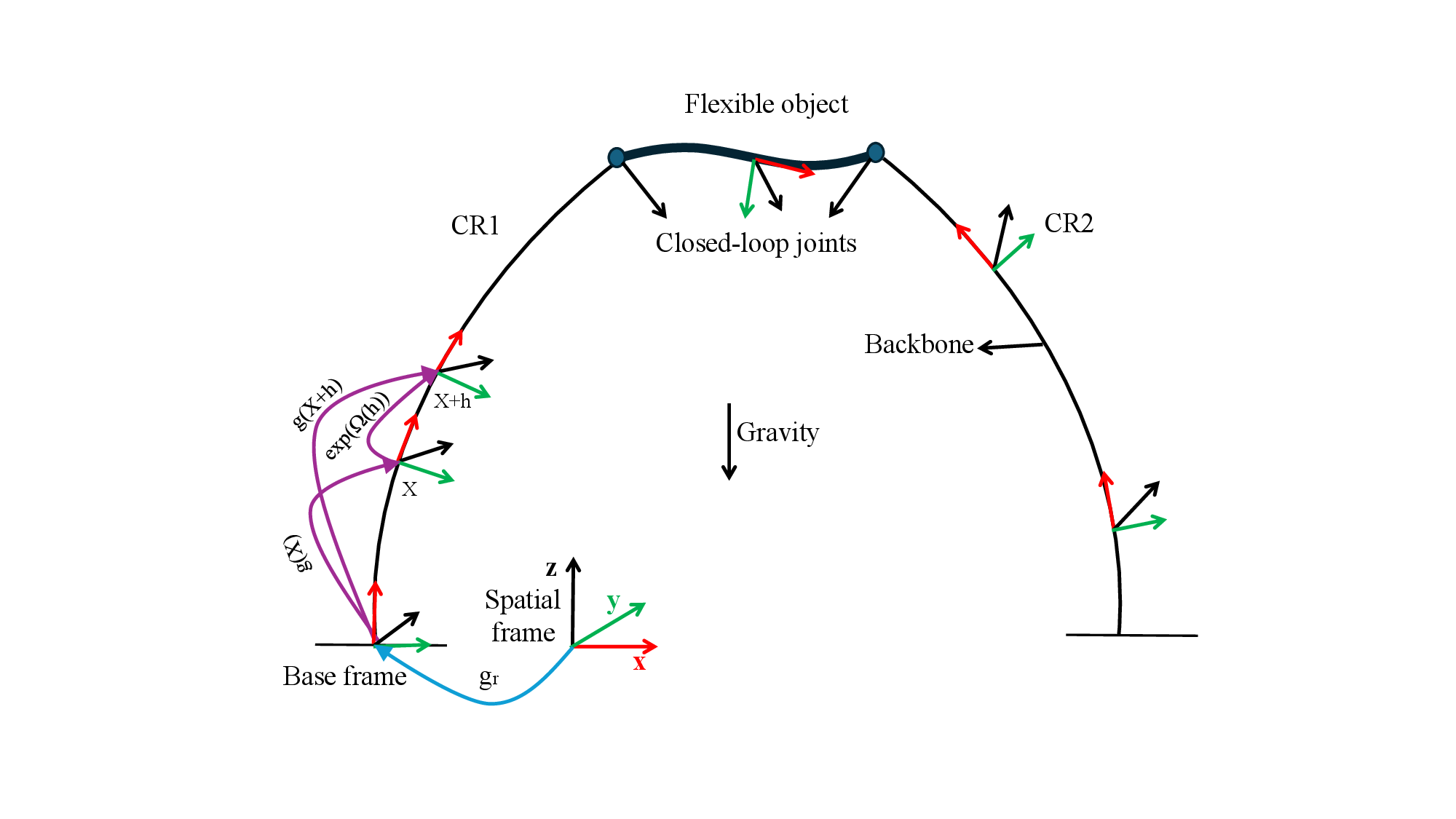}
\caption{Closed-loop configuration of the tendon-driven CCR system manipulating a flexible object, with backbone parameterization and reference frames.}
    \label{fig:schematic}
\end{figure}

\vspace{-1 em}

\subsection{GVS Kinematics}
A CR is modeled as a mechanically continuous system parameterized by the backbone coordinate {\small $X \in [0,L]$}, where {\small $X=0$} corresponds to the base frame and {\small $L$} denotes the backbone length. As shown in Fig.~\ref{fig:schematic},
a local body frame is attached to each cross section of the backbone. The pose of this frame with respect to the spatial frame
is described by the homogeneous transformation
{\small
\begin{equation}
\mathbf{g}(X)=
\begin{bmatrix}
\mathbf{R}(X) & \mathbf{p}(X) \\
\mathbf{0}^{\mathsf{T}} & 1
\end{bmatrix}
\in SE(3),
\label{eq:g_SE3}
\end{equation}}

\noindent where {\small$\mathbf{p}(X)\in\mathbb{R}^3$} denotes the position of the origin of the local
frame, and {\small$\mathbf{R}(X)\in SO(3)$} represents the orientation of the cross section
with respect to the spatial frame.

The kinematics are formulated on the Lie group {\small$SE(3)$}, and differentiation of the
configuration with respect to the backbone coordinate {\small$X$} and time {\small$t$} yields

{\small\begin{equation}
\mathbf{g}'=\mathbf{g}\widehat{\boldsymbol{\xi}},
\qquad
\dot{\mathbf{g}}=\mathbf{g}\widehat{\boldsymbol{\eta}},
\label{eq:g_kinematics}
\end{equation}}

\noindent where {\small $(\cdot)'$} and {\small $(\cdot)\dot{\;}$} denote derivatives with respect to {\small $X$} and {\small $t$},
respectively, and {\small $\widehat{(\cdot)}:\mathbb{R}^6\rightarrow\mathfrak{se}(3)$} is the
standard hat operator. The strain twist {\small $\boldsymbol{\xi}\in\mathbb{R}^6$} characterizes the local
deformation of the backbone, including bending, torsion, shear, and axial extension
per unit length. The body velocity twist {\small $\boldsymbol{\eta}\in\mathbb{R}^6$}
describes the instantaneous motion of each cross section expressed in the local
frame.
The solution of the body kinematic equation {\small $\mathbf{g}'=\mathbf{g}\widehat{\boldsymbol{\xi}}$},
illustrated in Fig.~\ref{fig:schematic}, over an interval of length $h$ can be written in Lie-group form as
{\small\begin{equation}
\mathbf{g}(X+h)=\mathbf{g}(X)\exp\!\big(\widehat{\boldsymbol{\Omega}}(h)\big),
\end{equation}}

\noindent where {\small $\boldsymbol{\Omega}(h)\in\mathbb{R}^6$} is the Magnus increment
associated with the strain twist over {\small $[X,X+h]$}. The increment {\small $\boldsymbol{\Omega}(h)$}
is computed using a fourth-order Magnus integrator based on Zanna collocation with two-stage Gauss
quadrature~\cite{murray2017mathematical,renda2020geometric}.

The equality of mixed partial derivatives,
{\small $(\dot{\mathbf{g}})'=(\mathbf{g}')^{\dot{}}$}, yields the compatibility condition
{\small
\begin{equation}
\boldsymbol{\eta}'=
\dot{\boldsymbol{\xi}}-
\mathrm{ad}_{\boldsymbol{\xi}}\boldsymbol{\eta},
\label{eq:compatibility}
\end{equation}
}

\noindent where {\small $\mathrm{ad}_{(\cdot)}\in\mathbb{R}^{6\times6}$} denotes the adjoint representation of 
{\small $\mathfrak{se}(3)$}. By using the identity {\small $(\mathrm{Ad}_{\mathbf{g}})'=\mathrm{Ad}_{\mathbf{g}}\mathrm{ad}_{\boldsymbol{\xi}}$}, the body velocity can be written analytically 

{\small
\begin{equation}
\boldsymbol{\eta}(X)=
\mathrm{Ad}_{\mathbf{g}^{-1}(X)}
\int_{0}^{X}
\mathrm{Ad}_{\mathbf{g}(s)}\dot{\boldsymbol{\xi}}(s)\,ds.
\label{eq:eta_analytic}
\end{equation}
} 

\noindent To obtain a finite-dimensional model, the strain field is approximated using a set
of basis functions,
{\small
\begin{equation}
\boldsymbol{\xi}(X)=
\mathbf{B}_{\boldsymbol{\xi}}(X)\mathbf{q}
+\boldsymbol{\xi}^{\ast}(X),
\label{eq:strain_basis}
\end{equation}
}

\noindent where {\small $\mathbf{B}_{\boldsymbol{\xi}}(X)\in\mathbb{R}^{6\times n}$} is the strain basis
matrix and {\small $\mathbf{q}\in\mathbb{R}^n$} denotes the generalized coordinates, with {\small $n$} representing the total number of degrees of freedom.
Substituting \eqref{eq:strain_basis} into \eqref{eq:eta_analytic} yields
{\small
\begin{equation}
\begin{aligned}
\boldsymbol{\eta}(X)=&
\mathrm{Ad}_{\mathbf{g}^{-1}(X)}
\int_{0}^{X}
\mathrm{Ad}_{\mathbf{g}(s)}
\mathbf{B}_{\boldsymbol{\xi}}(s)\,ds\;
\dot{\mathbf{q}} 
=
\mathbf{J}(\mathbf{q},X)\dot{\mathbf{q}},
\end{aligned}
\label{eq:geometric_jacobian}
\end{equation}
}

\noindent where {\small $\mathbf{J}(\mathbf{q},X)$} is the geometric Jacobian that maps the generalized
velocities to the body velocity along the backbone.

\vspace{-1 em}

\subsection{GVS Dynamics}
The dynamics of the CR are derived using Cosserat rod theory
formulated on the Lie algebra {\small $\mathfrak{se}(3)$}~\cite{renda2018discrete}. The strong form of the balance
of linear and angular momentum along the backbone is given by
{\small
\begin{equation}
\boldsymbol{\mathcal{M}}\dot{\boldsymbol{\eta}}
+
\mathrm{ad}^{*}_{\boldsymbol{\eta}}\,\boldsymbol{\mathcal{M}}\boldsymbol{\eta}
=
\boldsymbol{\mathcal{F}}_i'
+
\mathrm{ad}^{*}_{\boldsymbol{\xi}}\,\boldsymbol{\mathcal{F}}_i
+
\bar{\boldsymbol{\mathcal{F}}}_a
+
\bar{\boldsymbol{\mathcal{F}}}_e ,
\label{eq:cosserat_strong}
\end{equation}
}

\noindent where {\small $\boldsymbol{\mathcal{F}}_i$} denotes the internal wrench,
{\small $\bar{\boldsymbol{\mathcal{F}}}_a$} represents the distributed actuation wrench,
{\small $\bar{\boldsymbol{\mathcal{F}}}_e$} denotes the distributed external wrench, and
{\small $\boldsymbol{\mathcal{M}}\in\mathbb{R}^{6\times6}$} is the screw inertia matrix. The operator
{\small $\mathrm{ad}^{*}_{(\cdot)}$} denotes the coadjoint action on
{\small $\mathfrak{se}(3)$}~\cite{renda2020geometric}. A local body frame is attached to each cross section, with the {\small $X$} axis
aligned with the backbone direction and the {\small $Y$} and {\small $Z$} axes lying in the
symmetric cross-sectional plane, as shown in Fig. \ref{fig:schematic}. Under this choice of coordinates, the screw
inertia matrix is expressed as
{\small
\begin{equation}
\boldsymbol{\mathcal{M}} = \rho\,\mathrm{diag}(J_x, J_y, J_z, A, A, A),
\end{equation}
}

\noindent where {\small $\rho$} denotes the material density, {\small $A$} is the cross-sectional area, and
{\small $J_x$}, {\small $J_y$}, and {\small $J_z$} are the second moments of area associated with torsion
and bending of the backbone cross section.
The internal wrench {\small $\boldsymbol{\mathcal{F}}_i$} is formulated based on a
linear Kelvin--Voigt viscoelastic constitutive assumption~\cite{renda2014dynamic}, which captures both the
elastic response and the rate dependent dissipation of the continuum backbone.
Under this model, the internal wrench is related to the strain and strain rate
fields according to
{\small
\begin{equation}
\boldsymbol{\mathcal{F}}_i
=
\boldsymbol{\mathcal{K}}\boldsymbol{\xi}
+
\boldsymbol{\mathcal{D}}\,\dot{\boldsymbol{\xi}},
\label{eq:internal_wrench_kv}
\end{equation}
}

\noindent where the matrices {\small $\boldsymbol{\mathcal{K}}$} and {\small $\boldsymbol{\mathcal{D}}$} are the material
stiffness and viscous damping properties of the rod\cite{hussain2021compliant}.
The external loading acting on the CR includes distributed effects
due to gravity as well as concentrated loads arising from externally applied
forces or contact interactions. The resulting external wrench density is given by
{\small
\begin{equation}
\bar{\boldsymbol{\mathcal{F}}}_e
=
\boldsymbol{\mathcal{M}}
\,\mathrm{Ad}^{-1}_{\mathbf{g}_r\mathbf{g}}\,\boldsymbol{\mathcal{G}}
+
\delta(X-\bar{X})\,\boldsymbol{\mathcal{F}}_p,
\label{eq:external_wrench}
\end{equation}
}

\noindent where {\small
\(
\boldsymbol{\mathcal{G}}
\)
}represents the gravity twist in the spatial frame. The transformation {\small $\mathbf{g}_r\in SE(3)$}, as shown in Fig.\ref{fig:schematic}, maps quantities between the spatial
frame and the base frame of the continuum manipulator. The
operator {\small $\delta(\cdot)$} denotes the Dirac distribution, which localizes the
concentrated wrench {\small $\boldsymbol{\mathcal{F}}_p$} applied at the position
{\small $X=\bar{X}$} along the backbone.

\vspace{-0.5 em}

\subsection{Discrete Dynamics}
For the purpose of discretization, the strong form in
{\small (\ref{eq:cosserat_strong})} is reformulated as a weak statement using the
principle of virtual work~\cite{renda2018discrete}.
Let {\small $\delta\boldsymbol{\zeta}(X)\in\mathbb{R}^6$} denote an arbitrary virtual twist
field along the backbone. The weak form of the balance of momentum is then given by

{\small
\begin{equation}
\begin{aligned}
\int_{0}^{L}
\delta\boldsymbol{\zeta}^{\mathsf{T}}
(
&\boldsymbol{\mathcal{M}}\dot{\boldsymbol{\eta}}
+
\mathrm{ad}^{*}_{\boldsymbol{\eta}}\,\boldsymbol{\mathcal{M}}\boldsymbol{\eta}
- \boldsymbol{\mathcal{F}}_i'
-
\mathrm{ad}^{*}_{\boldsymbol{\xi}}\,\boldsymbol{\mathcal{F}}_i)
\,dX
= \\
&\int_{0}^{L}
\delta\boldsymbol{\zeta}^{\mathsf{T}}
(\bar{\boldsymbol{\mathcal{F}}}_a
+
\bar{\boldsymbol{\mathcal{F}}}_e)
\,dX .
\end{aligned}
\label{eq:weak_form}
\end{equation}
}

\noindent The weak form {\small \eqref{eq:weak_form}} is expressed in terms of the generalized coordinates using
{\small $\delta\boldsymbol{\zeta}=\mathbf{J}\delta\mathbf{q}$} and
{\small $\dot{\boldsymbol{\eta}}=\mathbf{J}\ddot{\mathbf{q}}+\dot{\mathbf{J}}\dot{\mathbf{q}}$}.
Substituting these relations into {\small \eqref{eq:weak_form}}, and incorporating the external load model {\small \eqref{eq:external_wrench}}, yields the final form of the generalized equation of motion:

{\small
\begin{equation}
\begin{aligned}
&\left[\int_{0}^{L}\mathbf{J}^{\mathsf{T}}\boldsymbol{\mathcal{M}}\mathbf{J}\,dX\right]\ddot{\mathbf{q}}
+
\left[\int_{0}^{L}\!\left(
\mathbf{J}^{\mathsf{T}}\boldsymbol{\mathcal{M}}\dot{\mathbf{J}}
+
\mathbf{J}^{\mathsf{T}}\mathrm{ad}^{*}_{\mathbf{J}\dot{\mathbf{q}}}
\boldsymbol{\mathcal{M}}\mathbf{J}
\right)\!dX\right]\dot{\mathbf{q}}-\\
&\int_{0}^{L}\mathbf{J}^{\mathsf{T}}
\left(
\boldsymbol{\mathcal{F}}_i'
+
\mathrm{ad}^{*}_{\boldsymbol{\xi}}\boldsymbol{\mathcal{F}}_i
\right)dX
=
\left[\int_{0}^{L}\mathbf{J}^{\mathsf{T}}
\bar{\boldsymbol{\mathcal{F}}}_a
dX\right]+
\mathbf{J}(\bar{X})^{\mathsf{T}}\boldsymbol{\mathcal{F}}_{p}\\
&+\left[\int_{0}^{L}\mathbf{J}^{\mathsf{T}}\boldsymbol{\mathcal{M}}
\mathrm{Ad}^{-1}_{\mathbf{g}}\,dX\right]
\mathrm{Ad}^{-1}_{\mathbf{g}_r}\boldsymbol{\mathcal{G}} .
\end{aligned}
\label{eq:generalized_dynamics}
\end{equation}
}

\noindent Collecting the integral terms, {\small (\ref{eq:generalized_dynamics})} reduces to classical form
{\small
\begin{equation}
\mathbf{M}\ddot{\mathbf{q}}
+\mathbf{C}\dot{\mathbf{q}}
+\mathbf{K}\mathbf{q}
+\mathbf{D}\dot{\mathbf{q}}
=\boldsymbol{\tau}
+\mathbf{F}_{\mathrm{ext}}
+\mathbf{F}_{g},
\label{eq:compact_dynamics}
\end{equation}
}

\noindent where {\small $\mathbf{M},\mathbf{C},\mathbf{K},\mathbf{D}\in\mathbb{R}^{n\times n}$} denote the generalized inertia, Coriolis, stiffness, and damping matrices, respectively, and {\small $\boldsymbol{\tau},\mathbf{F}_{\mathrm{ext}},\mathbf{F}_{g}\in\mathbb{R}^{n}$} denote the actuation input, the generalized external force vector, and the generalized gravitational force vector, respectively~\cite{renda2020geometric,hussain2021compliant}.

\vspace{-1 em}

\subsection{Closed-Chain Dynamics}
The manipulation task addressed in this work inherently imposes
loop closure constraints arising from the mutual coupling between the CRs and
the manipulated flexible object. To systematically derive the equations of motion of the resulting closed-chain
system, we adopt a Lie group  formulation for CRs as developed
in \cite{armanini2021discrete}. The loop closure conditions are expressed at the velocity level in Pfaffian form
as a set of linear constraints on the generalized velocities. Specifically, the
closed-chain kinematics are described by

{\small
\begin{equation}
\mathbf{A}(\mathbf{q})\,\dot{\mathbf{q}}=\mathbf{0},
\label{eq:pfaffian_constraint}
\end{equation}}

\noindent where {\small $\mathbf{A}(\mathbf{q})\in\mathbb{R}^{n_c\times n}$} denotes the constraint
Jacobian and {\small $n_c$} is the number of constraints. {\small $\mathbf{A}$} is constructed by enforcing the kinematic compatibility
at each closed loop joint connecting two frames on bodies {\small $A$} and {\small $B$}. For the
{\small $i$}-th joint, the motion of each joint frame is described by body Jacobians
{\small $\mathbf{J}_{A_i}(\mathbf{q}),\mathbf{J}_{B_i}(\mathbf{q})\in\mathbb{R}^{6\times n}$},
which map the generalized velocity {\small $\dot{\mathbf{q}}$} to the corresponding body
twists, both expressed in the same joint coordinate frame. The relative twist
across the closed loop joint is therefore given by
{\small
\begin{equation}
\boldsymbol{\eta}_{i}
=
\big(\mathbf{J}_{A_i}(\mathbf{q})-\mathbf{J}_{B_i}(\mathbf{q})\big)\dot{\mathbf{q}}.
\end{equation}}

Closed-loop constraints restrict this relative motion by eliminating twist
components that are incompatible with the joint type. These restrictions are
encoded through a matrix {\small$\mathbf{B}_{p,i}\in\mathbb{R}^{6\times n_{ci}}$}, whose
columns span the constraint wrench subspace associated with the joint. By the
principle of dual orthogonality, allowable relative motion must be orthogonal to
the constraint wrenches, which yields the Pfaffian constraint

{\small
\begin{equation}
\mathbf{B}_{p,i}^{\mathsf{T}}\boldsymbol{\eta}_{i}=0.
\end{equation}}

\noindent Substituting the Jacobian expression of the relative twist and stacking the Pfaffian constraints for all closed loop joints yields the global constraint Jacobian
{\small
\begin{equation}
\mathbf{A}(\mathbf{q})
=
\begin{bmatrix}
\mathbf{B}_{p,i}^{\mathsf{T}}
\big(\mathbf{J}_{A_i}(\mathbf{q})-\mathbf{J}_{B_i}(\mathbf{q})\big)
\end{bmatrix}_{i=1}^{n_c}
\in \mathbb{R}^{n_c \times n}.
\label{eq:A_matrix_form}
\end{equation}
}

The internal reaction forces and moments required to enforce the closed-chain constraints are introduced through the Lagrange multiplier vector {\small $\boldsymbol{\lambda}\in\mathbb{R}^{n_c}$}. Accordingly, the generalized dynamics is
augmented by the constraint forces as
{\small
\begin{equation}
\begin{aligned}
\mathbf{M}\ddot{\mathbf{q}}
+
\mathbf{C}\dot{\mathbf{q}}
+
\mathbf{K}\mathbf{q}
+
\mathbf{D}\dot{\mathbf{q}}
=
\boldsymbol{\tau}
+
\mathbf{F_{ext}}
+
\mathbf{F_{g}}
+
\mathbf{A}^{\mathsf{T}}\boldsymbol{\lambda}.
\label{eq:compact_dynamics_constrain}
\end{aligned}
\end{equation}
}

The coupled system {\small \eqref{eq:compact_dynamics_constrain}} can be
solved explicitly for the Lagrange multipliers
{\small $\boldsymbol{\lambda}$} and substituted back into the dynamics to
obtain an equivalent closed-chain equation expressed solely in
terms of the generalized coordinates. The constrained dynamics are expressed using an orthogonal projection onto the null space of the constraint Jacobian, {\small $\mathbf{A}(\mathbf{q})$}, as formalized below.

\begin{lemma}{\normalfont\cite{aghili2011projection}}
\label{lem:P_orth}

Let {\small $\mathbf{A}(\mathbf{q})\in\mathbb{R}^{n_c\times n}$} denote the
constraint Jacobian and assume it has full row rank. Define
{\small
\begin{align}
\mathbf{P}(\mathbf{q})
=
\mathbf{I}
-
\mathbf{A}^{+}(\mathbf{q})\,\mathbf{A}(\mathbf{q}),
\label{eq:P_def}
\end{align}
}

\noindent where {\small $\mathbf{A}^{+}$} is the Moore–Penrose inverse.
Then {\small $\mathbf{P}$} is an orthogonal projector satisfying
{\small
\begin{align}
\mathbf{P}^{\mathsf{T}}=\mathbf{P},\qquad
\mathbf{P}^2=\mathbf{P},
\end{align}
}
and it eliminates the constraint directions
{\small
\begin{align}
\mathbf{A}\mathbf{P}=\mathbf{0},\qquad
\mathbf{P}\mathbf{A}^{\mathsf{T}}=\mathbf{0}.
\end{align}
}
\end{lemma}

Premultiplying the constrained dynamics {\small\eqref{eq:compact_dynamics_constrain}} by {\small $\mathbf{P}$} eliminates
the constraint forces since
{\small $\mathbf{P}\mathbf{A}^{\mathsf{T}}\boldsymbol{\lambda}=\mathbf{0}$}. The
resulting projected dynamics is given by 
{\small
\begin{align}
\mathbf{P}
\Big(
\mathbf{M}\ddot{\mathbf{q}}
+
\mathbf{C}\dot{\mathbf{q}}
+
\mathbf{K}\mathbf{q}
+
\mathbf{D}\dot{\mathbf{q}}
\Big)
=
\mathbf{P}
\Big(
\boldsymbol{\tau}
+
\mathbf{F_{ext}}
+
\mathbf{F}_{g}
\Big),
\label{eq:projected_dynamics}
\end{align}
}

\noindent which evolves entirely in the tangent space of the constraint
manifold.

\vspace{-0.5 em}

\section{Adaptive Control}
This section develops a projected adaptive control framework for the closed-chain CCR system. The adaptive controller is designed to compensate for parametric uncertainties while satisfying the closed-chain constraints.

\vspace{-1 em}
\subsection{Inertial Parameterization}
To develop a model-based adaptive controller, the dynamic equation in
{\small \eqref{eq:compact_dynamics_constrain}} is exploited through its linear dependence on a set
of constant physical parameters. In particular, the screw inertia model
{\small\begin{equation}
\boldsymbol{\mathcal{M}} = \rho\,\mathrm{diag}(J_x, J_y, J_z, A, A, A)
\label{eq:M_screw_param}
\end{equation}}

\noindent is adopted, and a reduced set of four parameters per soft link, including each CR and the flexible object, is
treated as unknown constants to be estimated online. For the {\small$k$}-th soft link, the parameter vector is defined as
{\small
\begin{equation}
\boldsymbol{\theta}_k=
\begin{bmatrix}
\theta_{x,k}\!&\!\theta_{y,k}\!&\!\theta_{z,k}\!&\!\theta_{A,k}
\end{bmatrix}^{\!{\mathsf{T}}}
=\rho_k
\begin{bmatrix}
 J_{x,k}\ J_{y,k}\ J_{z,k}\ A_k
\end{bmatrix}^{\!{\mathsf{T}}}.
\label{eq:theta_i_def}
\end{equation}
}

For a system composed of {\small $N$} soft links, the full inertial parameter vector is
constructed by stacking the individual link parameters as
{\small
\begin{equation}
\boldsymbol{\theta}
=
\begin{bmatrix}
\boldsymbol{\theta}_1^{{\mathsf{T}}}\!&\!
\boldsymbol{\theta}_2^{{\mathsf{T}}}\!&\!
\cdots\!&\!
\boldsymbol{\theta}_N^{{\mathsf{T}}}
\end{bmatrix}^{{\mathsf{T}}}
\in\mathbb{R}^{4N}.
\label{eq:theta_global_def}
\end{equation}
}

The screw inertia matrix defined in {\small \eqref{eq:M_screw_param}} contributes explicitly
to the generalized mass matrix {\small $\mathbf{M}$}, the Coriolis matrix
{\small $\mathbf{C}$}, and the gravity force
{\small $\mathbf{F_{g}}$}. To express these terms in a form that is linear with
respect to the unknown parameter vector {\small $\boldsymbol{\theta}$}, the screw inertia
matrix is decomposed into a weighted sum of constant
basis matrices. Specifically, the following selector matrices are introduced:
{\small
\begin{equation}
\begin{aligned}
\mathbf{E}_x &= \mathrm{diag}(1,0,0,0,0,0),\;
\mathbf{E}_y = \mathrm{diag}(0,1,0,0,0,0),\\
\mathbf{E}_z &= \mathrm{diag}(0,0,1,0,0,0),\;
\mathbf{E}_A = \mathrm{diag}(0,0,0,1,1,1).
\end{aligned}
\label{eq:E_matrices}
\end{equation}
}

Using the selector matrices in {\small \eqref{eq:E_matrices}}, the screw inertia matrix
{\small \eqref{eq:M_screw_param}} for the {\small $k$}-th soft link can be written as
{\small
\begin{equation}
\boldsymbol{\mathcal{M}}_k
=
\theta_{x,k} \mathbf{E}_x
+
\theta_{y,k} \mathbf{E}_y
+
\theta_{z,k} \mathbf{E}_z
+
\theta_{A,k} \mathbf{E}_A.
\label{eq:M_linear_decomp}
\end{equation}
}

\noindent The inertia description of the full system is then obtained by assembling the screw inertia matrices of all soft links.

\subsection{Linear Parameterization of the Dynamics}
The generalized inertia contribution is expressed in a form that is
linear with respect to the unknown parameter vector
{\small $\boldsymbol{\theta}$}. By substituting the decomposition of the
screw inertia matrix into the definition of the generalized mass matrix,
the inertia term can be written as a weighted sum of configuration-dependent
integrals, where the unknown physical parameters appear as constant scalar
coefficients as

{\small
\begin{equation}
\mathbf{M}\ddot{\mathbf{q}}
=
\sum_{i\in\{x,y,z,A\}}
\theta_i
\left(
\int_{0}^{L}
\mathbf{J}^{\mathsf{T}}\mathbf{E}_i\mathbf{J}\,dX
\right)
\ddot{\mathbf{q}}
=
\mathbf{Y}_{M}\boldsymbol{\theta},
\label{eq:Mddq_compact_all}
\end{equation}
}

\noindent
where {\small $\mathbf{Y}_{M}$} denotes the inertia regressor matrix and
{\small $\boldsymbol{\theta}$} collects the constant inertial parameters.
A similar decomposition is applied to the Coriolis term
{\small $\mathbf{C}\dot{\mathbf{q}}$}.
By substituting the parameter decomposition into the generalized dynamics,
the Coriolis contribution can be written in parameter-linear form as

{\small
\begin{equation}
\mathbf{C}\dot{\mathbf{q}}
=
\smash{\sum_{i\in\{x,y,z,A\}}}\!
\theta_i(
\int_{0}^{L}\!
\mathbf{J}^{\mathsf{T}}(
\mathbf{E}_i\dot{\mathbf{J}}
+\mathrm{ad}^{*}_{\mathbf{J}\dot{\mathbf{q}}}\mathbf{E}_i\mathbf{J}
)\,dX
)\
\dot{\mathbf{q}}
=
\mathbf{Y}_{C}\boldsymbol{\theta},
\end{equation}
}

\noindent
where {\small $\mathbf{Y}_{C}$} denotes the Coriolis regressor matrix.
The gravitational term in \eqref{eq:compact_dynamics} admits a parameter-linear representation given by
{\small
\begin{equation}
\mathbf{F}_g
=
\smash{\sum_{i\in\{x,y,z,A\}}}\!
\theta_i\!
\left(
\int_{0}^{L}\!
\mathbf{J}^{\mathsf{T}}\mathbf{E}_i
\mathrm{Ad}^{-1}_{\mathbf{g}}\,dX
\right)
\mathrm{Ad}^{-1}_{\mathbf{g}_r}\boldsymbol{\mathcal{G}}
=
\mathbf{Y}_g\boldsymbol{\theta},
\label{eq:Fg_compact}
\end{equation}
}

\noindent
where {\small $\mathbf{Y}_g$} denotes the gravity regressor matrix.

Rearranging \eqref{eq:compact_dynamics_constrain} and grouping the inertial parameters yields
{\small
\begin{equation}
\begin{aligned}
\mathbf{M}\ddot{\mathbf{q}}
+\mathbf{C}\dot{\mathbf{q}}
-\mathbf{F}_g
=
\mathbf{Y}\boldsymbol{\theta}
=
\boldsymbol{\tau}
-\mathbf{K}\mathbf{q}
-\mathbf{D}\dot{\mathbf{q}}
+\mathbf{F}_{\mathrm{ext}}
+\mathbf{A}^{\mathsf{T}}\boldsymbol{\lambda}.
\end{aligned}
\label{eq:regressor_rearranged}
\end{equation}
}

\noindent
The combined regressor is defined as
{\small
\begin{equation}
\mathbf{Y}(\mathbf{q},\dot{\mathbf{q}},\ddot{\mathbf{q}})
=
\mathbf{Y}_M
+
\mathbf{Y}_C
-
\mathbf{Y}_g.
\label{eq:Y_total_def}
\end{equation}
}

Applying the orthogonal projector {\small$\mathbf{P}(\mathbf{q})$}
defined in Lemma{\small~\ref{lem:P_orth}} to the parameter linear
form {\small\eqref{eq:regressor_rearranged}} yields the constraint consistent
dynamics

{\small
\begin{equation}
\mathbf{P}\Big(
\mathbf{M}\ddot{\mathbf{q}}
+\mathbf{C}\dot{\mathbf{q}}
-\mathbf{F}_g
\Big)
=
\mathbf{P}\mathbf{Y}(\mathbf{q},\dot{\mathbf{q}},\ddot{\mathbf{q}})\boldsymbol{\theta}
=
\mathbf{P}\Big(
\boldsymbol{\tau}
-\mathbf{K}\mathbf{q}
-\mathbf{D}\dot{\mathbf{q}}
+\mathbf{F}_{\mathrm{ext}}
\Big).
\label{eq:projected_full_dyn}
\end{equation}
}

\subsection{Reference Signals and Sliding Variable}
Let the task-space tracking error be defined as
{\small
\begin{align}
\mathbf{e}
&=
\mathbf{x}_d-\mathbf{x},
\qquad
\dot{\mathbf{e}}
=
\dot{\mathbf{x}}_d-\dot{\mathbf{x}},
\label{eq:e_def}
\end{align}
}

\noindent where {\small $\mathbf{e},\dot{\mathbf{e}}\in\mathbb{R}^3$}. The vector
{\small $\mathbf{x}\in\mathbb{R}^3$} denotes the position of a point of
interest on the CCR, modeled as a single closed-chain system. In this
study, {\small $\mathbf{x}$} is chosen as the midpoint position of the
flexible object. The velocity {\small $\dot{\mathbf{x}}\in\mathbb{R}^3$}
corresponds to the translational component of the body
velocity {\small $\boldsymbol{\eta}(X)\in\mathbb{R}^6$} obtained from the
GVS formulation. Here, {\small $\mathbf{x}_d\in\mathbb{R}^3$} and {\small $\dot{\mathbf{x}}_d\in\mathbb{R}^3$} denote the desired task-space trajectory and the corresponding desired velocity, respectively.
A reference task-space velocity is defined as
{\small
\begin{align}
\dot{\mathbf{x}}_r
=
\dot{\mathbf{x}}_d
+\mathbf{\Lambda}\mathbf{e},
\label{eq:xr_def}
\end{align}
}

\noindent where {\small $\mathbf{\Lambda}\in\mathbb{R}^{3\times3}$} is a symmetric positive definite matrix, which defines first-order reference error dynamics.
The corresponding joint-space reference velocity is obtained via differential inverse kinematics using the geometric Jacobian defined in {\small \eqref{eq:geometric_jacobian}}. To enforce consistency with the closed-chain constraints, the reference velocity is projected onto the null space of the constraint Jacobian:
{\small
\begin{align}
\dot{\mathbf{q}}_r
=
\mathbf{P}\,\mathbf{J}^{\dagger}(\mathbf{q})\,\dot{\mathbf{x}}_r,
\end{align}
}

\noindent where {\small $\mathbf{J}^{\dagger}(\mathbf{q})$} denotes the Jacobian pseudoinverse.

\begin{assumption}
\label{ass:constraint_consistent}
The reference joint velocity is constructed such that it satisfies the Pfaffian velocity constraint
{\small
\begin{equation}
\mathbf{A} \dot{\mathbf{q}}_r = \mathbf{0},
\label{eq:ref_constraint}
\end{equation}
}
which is ensured by the projection-based reference generation.
\end{assumption}

The sliding variable is defined as
{\small
$
\mathbf{s}=
\dot{\mathbf{q}} - \dot{\mathbf{q}}_r.
\label{eq:s_def}
$}
Since both the actual and reference velocities satisfy the
Pfaffian constraint,
{\small $\mathbf{A}\dot{\mathbf{q}}= \mathbf{A}\dot{\mathbf{q}}_r=\mathbf{0}$},
it follows that
{\small
\begin{equation}
\mathbf{A}\mathbf{s}=\mathbf{0}
\;\;\Longrightarrow\;\;
\mathbf{s}\in\mathcal{N}\!\left(\mathbf{A}(\mathbf{q})\right),
\end{equation}
}

\noindent where $\mathcal{N}(\cdot)$ denotes the null space of the matrix.
From Lemma \ref{lem:P_orth}, the orthogonal projector
{\small $\mathbf{P}$}
maps onto this null space, yielding
{\small
\begin{equation}
\mathbf{P}\mathbf{s}=\mathbf{s}.
\label{eq:Ps_equals_s}
\end{equation}
}

From the linear parameterization property established in
{\small \eqref{eq:projected_full_dyn}}, the projected reference dynamics satisfies
{\small
\begin{equation}
\mathbf{P}
\Big(
\mathbf{M}\ddot{\mathbf{q}}_r
+
\mathbf{C}\dot{\mathbf{q}}_r
-
\mathbf{F}_g
\Big)
=
\mathbf{P}
\mathbf{Y}(\mathbf{q},\dot{\mathbf{q}},
\dot{\mathbf{q}}_r,\ddot{\mathbf{q}}_r)
\boldsymbol{\theta}.
\label{eq:proj_ref_linear}
\end{equation}
}

If {\small $\mathbf{s}\to\mathbf{0}$} and the Jacobian
{\small $\mathbf{J}$} remains full rank, then from {\small \eqref{eq:xr_def}} and sliding variable,
the error dynamics is
{\small
\[
\dot{\mathbf{e}}+\mathbf{\Lambda}\mathbf{e}\to\mathbf{0},
\]
}

\noindent which implies exponential convergence
{\small $\mathbf{e}(t)\to\mathbf{0}$}.
Subtracting {\small \eqref{eq:proj_ref_linear}} from
{\small \eqref{eq:projected_full_dyn}} and using the sliding
variable definition {\small $\dot{\mathbf{s}}=\ddot{\mathbf{q}}-\ddot{\mathbf{q}}_r$}
yields the projected error dynamics
{\small
\begin{align}
\mathbf{P}\Big(
\mathbf{M}\dot{\mathbf{s}}
+
\mathbf{C}\mathbf{s}
\Big)
&=
\mathbf{P}\Big(
\boldsymbol{\tau}
-
\mathbf{K}\mathbf{q}
-
\mathbf{D}\dot{\mathbf{q}}
+
\mathbf{F}_{\mathrm{ext}}
\Big)
\nonumber \\
&\quad
-
\mathbf{P}\mathbf{Y}\big(
\mathbf{q},\dot{\mathbf{q}},
\dot{\mathbf{q}}_r,\ddot{\mathbf{q}}_r
\big)\boldsymbol{\theta}.
\label{eq:s_dyn_open}
\end{align}
}

\subsection{Projected Adaptive Control Law}

Define the projected control input
{\small
$
\boldsymbol{\tau}_p
:=
\mathbf{P}(\mathbf{q})\,\boldsymbol{\tau},
\label{eq:tau_p_def}
$
}
which represents the constraint-consistent component of the actuation.
The projected adaptive control law is given by
{\small
\begin{align}
\boldsymbol{\tau}_p
&=
\mathbf{P}\mathbf{K}\mathbf{q}
+
\mathbf{P}\mathbf{D}\dot{\mathbf{q}}
-
\mathbf{P}\mathbf{F}_{\mathrm{ext}}
+
\mathbf{P}\mathbf{Y}(\mathbf{q},\dot{\mathbf{q}},
\dot{\mathbf{q}}_r,\ddot{\mathbf{q}}_r)
\hat{\boldsymbol{\theta}}
-
\mathbf{K}_s\mathbf{s},
\label{eq:u_adapt}
\end{align}
}

\noindent where {\small $\mathbf{K}_s\in\mathbb{R}^{n\times n}$} is a symmetric positive definite matrix, {\small $\mathbf{K}_s\succ\mathbf{0}$}, and {\small $\hat{\boldsymbol{\theta}}$} denotes the estimate of the unknown constant parameter vector {\small $\boldsymbol{\theta}$}.
Define the parameter estimation error
{\small
$
\tilde{\boldsymbol{\theta}}
:=
\hat{\boldsymbol{\theta}}-\boldsymbol{\theta}
\label{eq:theta_tilde}
$}
 and substituting the control law {\small \eqref{eq:u_adapt}} into
the projected dynamics {\small \eqref{eq:s_dyn_open}}
yields the closed-loop error system
{\small
\begin{align}
\mathbf{P}\Big(
\mathbf{M}\dot{\mathbf{s}}
+
\mathbf{C}\mathbf{s}
\Big)
=
-\mathbf{K}_s\mathbf{s}
+
\mathbf{P}\mathbf{Y}(\cdot)\tilde{\boldsymbol{\theta}},
\label{eq:s_dyn_closed}
\end{align}
}

\noindent where {\small $\mathbf{Y}(\cdot)
=\mathbf{Y}(\mathbf{q},\dot{\mathbf{q}},
\dot{\mathbf{q}}_r,\ddot{\mathbf{q}}_r)$}.

\begin{lemma}
\label{lem:skew}
Consider the dynamics
{\small \eqref{eq:generalized_dynamics}}–{\small \eqref{eq:compact_dynamics}},
where {\small $\mathbf{M}$} and
{\small $\mathbf{C}$}
are defined from the variational formulation.
Then the matrix
{\small $\dot{\mathbf{M}}-2\mathbf{C}$}
is skew-symmetric, and consequently
{\small
\begin{equation}
\mathbf{s}^{\mathsf{T}}\big(\dot{\mathbf{M}}-2\mathbf{C}\big)\mathbf{s}=0,
\qquad
\forall\,\mathbf{s}\in\mathbb{R}^n.
\end{equation}
}
\end{lemma}
This follows from the kinetic-energy structure of the GVS model and the
coadjoint operator identities underlying
{\small \eqref{eq:generalized_dynamics}}.

\begin{theorem}
\label{thm:proj_adapt_stab}
Consider the control law {\small \eqref{eq:u_adapt}} with the
adaptation law
{\small
\begin{equation}
\dot{\hat{\boldsymbol{\theta}}}
=
-\mathbf{\Gamma}\mathbf{Y}^{\mathsf{T}}(\cdot)\mathbf{P}\mathbf{s},
\qquad
\mathbf{\Gamma}\succ\mathbf{0}.
\label{eq:theta_dot}
\end{equation}
}
Then all closed-loop signals remain bounded and
{\small $\mathbf{s}(t)\to\mathbf{0}$} as
{\small $t\to\infty$}.
\end{theorem}

\begin{proof}
Consider the Lyapunov function
{\small
\begin{equation}
V
=
\frac12 \mathbf{s}^{\mathsf{T}}\mathbf{M}\mathbf{s}
+
\frac12
\tilde{\boldsymbol{\theta}}^{\mathsf{T}}
\mathbf{\Gamma}^{-1}
\tilde{\boldsymbol{\theta}},
\label{eq.Lyapunov function}
\end{equation}
}

\noindent which is positive definite since
{\small $\mathbf{M}\succ0$} and
{\small $\mathbf{\Gamma}\succ0$}.
Its derivative is
{\small
\begin{equation}
\dot V
=
\frac12\mathbf{s}^{\mathsf{T}}\dot{\mathbf{M}}\mathbf{s}
+
\mathbf{s}^{\mathsf{T}}\mathbf{M}\dot{\mathbf{s}}
+
\tilde{\boldsymbol{\theta}}^{\mathsf{T}}
\mathbf{\Gamma}^{-1}
\dot{\hat{\boldsymbol{\theta}}},
\end{equation}
}
\noindent where the true parameters are constant.
Using {\small \eqref{eq:s_dyn_closed}} and
{\small $\mathbf{P}\mathbf{s}=\mathbf{s}$},
{\small
$
\mathbf{s}^{\mathsf{T}}\mathbf{M}\dot{\mathbf{s}}
=
-\mathbf{s}^{\mathsf{T}}\mathbf{K}_s\mathbf{s}
+
\mathbf{s}^{\mathsf{T}}\mathbf{P}\mathbf{Y}(\cdot)
\tilde{\boldsymbol{\theta}}
-
\mathbf{s}^{\mathsf{T}}\mathbf{C}\mathbf{s}.
$}
\noindent Invoking Lemma~{\small \ref{lem:skew}},
{\small
$
\dot V
=
-\mathbf{s}^{\mathsf{T}}\mathbf{K}_s\mathbf{s}
+
\mathbf{s}^{\mathsf{T}}\mathbf{P}\mathbf{Y}(\cdot)
\tilde{\boldsymbol{\theta}}
+
\tilde{\boldsymbol{\theta}}^{\mathsf{T}}
\mathbf{\Gamma}^{-1}
\dot{\hat{\boldsymbol{\theta}}}.
$}

\noindent With the adaptation law
{\small
\begin{equation}
\dot{\hat{\boldsymbol{\theta}}}
=
-\mathbf{\Gamma}\mathbf{Y}^{\mathsf{T}}(\cdot)\mathbf{P}\mathbf{s},
\end{equation}
}

\noindent the cross terms cancel, yielding
{\small
\begin{equation}
\dot V
=
-\mathbf{s}^{\mathsf{T}}\mathbf{K}_s\mathbf{s}
\le 0.
\end{equation}
}

Thus {\small $\dot V\le 0$}, implying that {\small $V(t)$} is nonincreasing and bounded for all {\small $t\ge 0$}. Since {\small $V$} is positive definite in {\small $\mathbf{s}$} and {\small $\tilde{\boldsymbol{\theta}}$}, both signals remain bounded. Moreover, from {\small $\dot V=-\mathbf{s}^{\mathsf{T}}\mathbf{K}_s\mathbf{s}$} with {\small $\mathbf{K}_s\succ0$}, it follows that {\small $\mathbf{s}\in\mathcal{L}_2$}, where $\mathcal{L}_2$ denotes the set of square-integrable signals, corresponding to signals with finite energy over time, satisfying $\int_0^\infty \|\mathbf{s}(t)\|^2\,dt < \infty$. Because the closed-loop dynamics ensure that {\small $\dot{\mathbf{s}}$} is bounded, Barbalat’s lemma implies that {\small $\mathbf{s}(t)\to\mathbf{0}$} as {\small $t\to\infty$}.

\end{proof}

\vspace{-1.5 em}
\section{Simulation Results}
\label{sec:sim_results}

For the simulation studies, two tendon-driven CRs are rigidly connected through a flexible object, forming a closed-chain system. The structural parameters are summarized in Table~\ref{tab:structural_parameters}.
The reported density of the CRs corresponds to the backbone material only. To ensure consistency between the simulation parameters and the physical system, an effective distributed density {\small $\rho_{\mathrm{eff}}$} is introduced to capture the combined mass of the backbone, tendons, and spacer disks. This is computed as {\small $\rho_{\mathrm{eff}} = m_{\mathrm{total}}/(A L)$}, yielding {\small $\rho_{\mathrm{eff}} = 39317~\mathrm{kg/m^3}$}.

\begin{table}
\centering
\caption{Physical Parameters of the CCR}
\label{tab:structural_parameters}
\setlength{\tabcolsep}{3pt}
\renewcommand{\arraystretch}{1.0}
\begin{tabular}{lll}
\toprule
\textbf{Component} & \textbf{Parameter} & \textbf{Value} \\
\midrule
\multirow{5}{*}{\shortstack{Continuum Arms\\(Spring steel)}}
& Length & $L = 0.4~\mathrm{m}$ \\
& Density & $\rho= 7800~\mathrm{kg/m^3}$ \\
& Young's Modulus & $E = 120~\mathrm{GPa}$ \\
& Poisson's Ratio & $\nu = 0.3$ \\
& Radius & $r= 0.9~\mathrm{mm}$ \\
\midrule
\multirow{5}{*}{\shortstack{Flexible Object\\(Nitinol)}}
& Length & $L = 0.22~\mathrm{m}$ \\
& Density & $\rho = 6450~\mathrm{kg/m^3}$ \\
& Young's Modulus & $E = 50~\mathrm{GPa}$ \\
& Poisson's Ratio & $\nu = 0.3$ \\
& Radius & $r = 0.8~\mathrm{mm}$ \\
\bottomrule
\end{tabular}
\end{table}

For numerical implementation, each CR is modeled as a single soft link with one division. Linear bending in two orthogonal directions is considered for each robot, resulting in {\small $4$} generalized coordinates per CR. The flexible object is described by {\small $2$} deformation coordinates associated with bending in two orthogonal directions, together with {\small $6$} spatial motion of it. Therefore, the CCR is characterized by {\small $n=16$} generalized coordinates. Two closed-loop
joints are defined at the connection points between the distal
tips of the continuum arms and the two ends of the flexible
object.
Numerical integration of the distributed dynamics for each CR and the flexible object is performed using Gaussian quadrature with {\small $7$} integration points per component.

Based on the above model, the proposed projected adaptive controller is evaluated in simulation under two scenarios: (i) task-space regulation; and (ii) task-space trajectory tracking. Since the two CRs and the flexible object form a mechanically coupled closed-chain, the CCR system is modeled as a single constrained dynamical system. Accordingly, the task variable {\small $\mathbf{x}$} may be defined at any point of interest along the closed-chain. In this work, the task-space output is chosen as the three-dimensional position of the midpoint of the manipulated object, denoted by {\small $\mathbf{x} = [x; y; z]^{\mathsf{T}}$}.

For all simulation studies, we report the time histories of the midpoint coordinates {\small $(x,y,z)$} relative to their desired values {\small $(x_d,y_d,z_d)$}, together with the Euclidean task-space error norm {\small $|\mathbf{e}|$}, where {\small $\mathbf{e}=\mathbf{x}_d-\mathbf{x}$}. We additionally show the corresponding tendon actuation inputs and the evolution of the online mass and inertial parameter estimates.

To assess robustness against modeling errors, all simulations are conducted in the presence of parametric uncertainty introduced in the dynamic model. Specifically, the true screw inertia matrix of each CR and the flexible object is scaled as
{\small
$
\boldsymbol{\mathcal{M}} = 0.7\,\boldsymbol{\mathcal{M}}_{\mathrm{nom}},
$}
representing a {\small$30\%$} deviation from the nominal inertia matrix used in the controller design. The uncertainty is embedded directly in the dynamic model, while the controller is initialized with nominal parameter values.

\vspace{-1 em}
\subsection{Task-Space Regulation}
The regulation performance of the manipulated object midpoint is shown in Fig.{\small ~\ref{fig:Sim_regulation_performance}}. 
The task-space coordinates are regulated toward the constant reference position 
{\small $\mathbf{x}_d = [315.6,\; 18.2,\; 331.6]^{\mathsf{T}}$} mm. As shown in Fig.~{\small \ref{fig:Sim_regulation_performance}}, all task-space coordinates converge to their desired values, and the tracking error decreases monotonically toward zero, confirming the effectiveness of the proposed projected adaptive controller.

The corresponding actuation inputs and parameter estimates are shown in
Fig.{\small ~\ref{fig:Sim_effort_adaptation_regulation}}. In the simulation study, we
report both the tendon tensions {\small $(u_1,u_2,u_3)$} and the corresponding tendon
displacements {\small $(S_1,S_2,S_3)$} as actuation signals. The reason for including both variables is to clarify their different
roles in modeling and implementation. The dynamic model of the CCR is formulated in terms of tendon tensions,
since the actuation enters Cosserat rod models and their discretized
representations, such as the GVS formulation, as generalized forces.
It has been observed in force-driven continuum systems derived from
Cosserat rod formulations~{\small\cite{DANESH2025105953}} that tendon tensions
may reach relatively high values during the initial stage of control,
therefore appropriate bounds on the tendon tensions must always be
considered. In the present study, the tendon tension is bounded within
{\small $+/-\,20~\mathrm{N}$}, and as shown in Fig.{\small ~\ref{fig:Sim_effort_adaptation_regulation}(a)},
the tension reaches this limit during the initial transient.

In the experimental setup, however, actuation through direct tension
control is less reliable due to motor and hardware limitations. For this reason, we use the dynamic
formulation to compute the required tendon lengths and then determine
the corresponding tendon displacements at each time step. These tendon displacements are
mapped to motor position commands and applied to the experimental
platform. Although the simulated tendon tensions may reach their prescribed
bounds during the initial transient, the corresponding tendon
displacements start from zero and evolve smoothly in positive or
negative directions, as illustrated in Fig.{\small ~\ref{fig:Sim_effort_adaptation_regulation}(b)}.
This confirms that tendon displacement is a more practical
and reliable actuation variable for experimental implementation.
Accordingly, both tendon tensions and tendon displacements are shown in
the simulation results to illustrate this behavior, while tendon
displacement is used to actuate the motors in the experimental study.

Furthermore, the online estimates of the mass and inertia parameters for
the two CRs and the manipulated flexible object evolve smoothly,
demonstrating adaptation to parametric uncertainty while preserving
closed-loop stability through the projected adaptive update law. In this
study, only the {\small $I_x$} parameter is shown in the figures, since the
cross sections are symmetric and satisfy {\small $I_x = 2 I_y = 2 I_z$}, and the
remaining inertia components are directly related.

\begin{figure}
  \centering
  \includegraphics[width=0.48\textwidth]{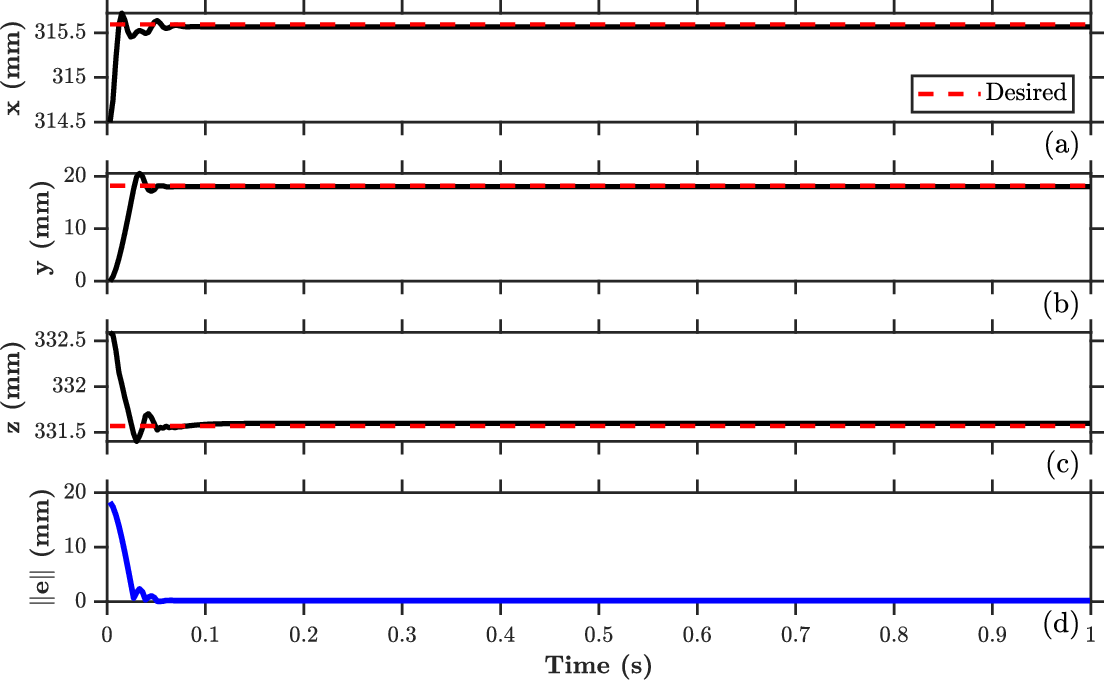}
\caption{Task-space regulation of the manipulated object midpoint in simulation. 
(a)--(c) $x$-, $y$-, and $z$-coordinates of the midpoint versus desired values, 
(d) Euclidean norm of the regulation error $\|\mathbf{e}\|$.}

  \label{fig:Sim_regulation_performance}
\end{figure}

\begin{figure}
  \centering
  \includegraphics[width=0.48\textwidth]{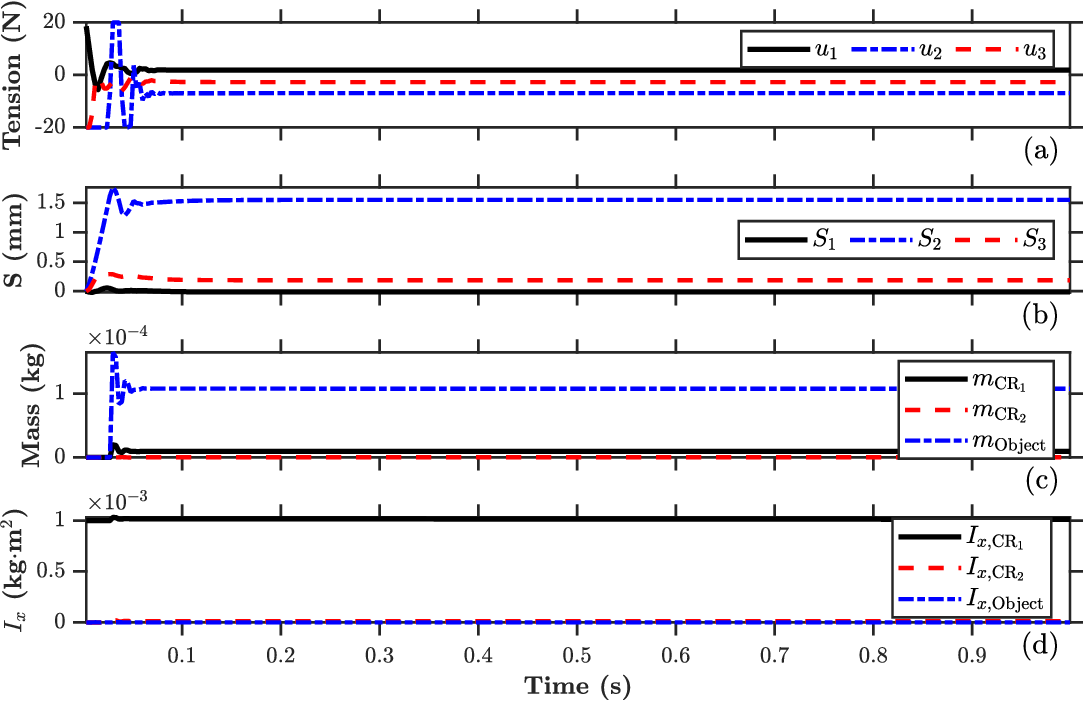}
\caption{Control inputs and adaptive parameter estimates during task-space regulation in simulation. 
(a) tendon tensions $u_1$, $u_2$, and $u_3$ (with saturation bounds of $[-20,\,20]~\mathrm{N}$). 
(b) tendon displacements $S_1$, $S_2$, and $S_3$. 
(c) estimated mass parameters of the two CRs and the manipulated flexible object. 
(d) estimated inertia parameter $I_x$.}
  \label{fig:Sim_effort_adaptation_regulation}
\end{figure}

\vspace{-1 em}
\subsection{Task-Space Trajectory Tracking}
Fig.{\small ~\ref{fig:Sim_tracking_performance}} presents the tracking performance of the
object midpoint for a time varying desired trajectory. The desired trajectory is defined as
{\small
$
i_d = \sum_{k=1}^{2} a_{ik}\sin(k t + \phi_{ik}) + c_i,
\ i \in \{x,z\},
y_d = a_{y}\sin(t + \phi_{y}) + c_y.$
}

\noindent where the constants are
{\footnotesize
$a_{x1}=0.000107, a_{x2}=0.00008, c_x=0.3146,
a_{y}=-0.01554,
c_y=-0.005173, 
a_{z1}=-0.000344, a_{z2}=-0.000257, c_z=0.3323,
\phi_{x1}=1.517, \phi_{x2}=1.458,
\phi_{y}=1.534, \phi_{z1}=1.533, \phi_{z2}=1.492.
$
}
\noindent The midpoint position tracks the desired motion in the {\small $x$}, {\small $y$}, and {\small $z$}
directions. The evolution of
{\small $\|\mathbf{e}\|$} confirms accurate tracking, and the inset over the
interval {\small $0$-$0.2~\mathrm{s}$} shows that the error norm decreases from
approximately {\small $20~\mathrm{mm}$} to zero within a short time, after which it
approaches zero.

The associated control effort and adaptive estimates are shown in
Fig.{\small ~\ref{fig:Sim_effort_adaptation}}. At the
beginning of the control process, the tendon tensions reach their
prescribed bounds; however, as shown in Fig.{\small ~\ref{fig:Sim_effort_adaptation}(b)},
the corresponding tendon displacements start from zero and evolve
smoothly. Simultaneously, the parameter estimates of the mass and {\small $I_x$}
terms for CRs and the manipulated object
adapt online, compensating for uncertainties in the coupled dynamics.

\begin{figure}
  \centering
  \includegraphics[width=0.48\textwidth]{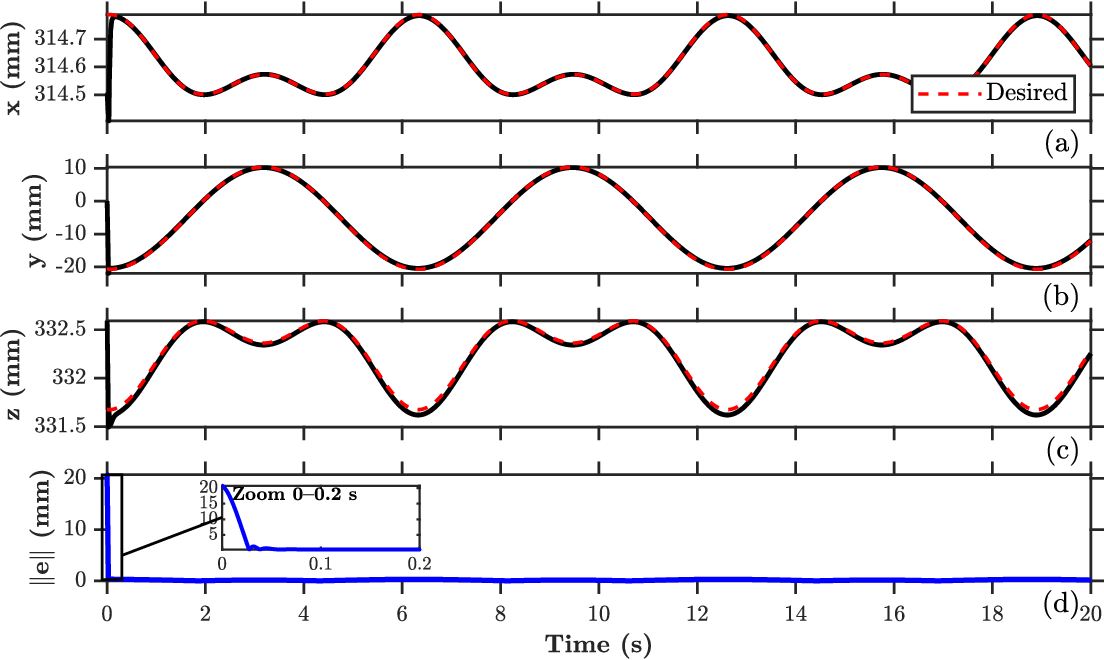}
  \caption{Task-space tracking performance of the manipulated object midpoint in simulation. 
(a)--(c) $x$-, $y$-, and $z$-coordinates of the midpoint, 
(d) norm of the tracking error $\|\mathbf{e}\|$. 
The inset in (d) shows a zoomed view over $0$--$0.2\,\mathrm{s}$.}
  \label{fig:Sim_tracking_performance}
\end{figure}

\begin{figure}
  \centering
  \includegraphics[width=0.48\textwidth]{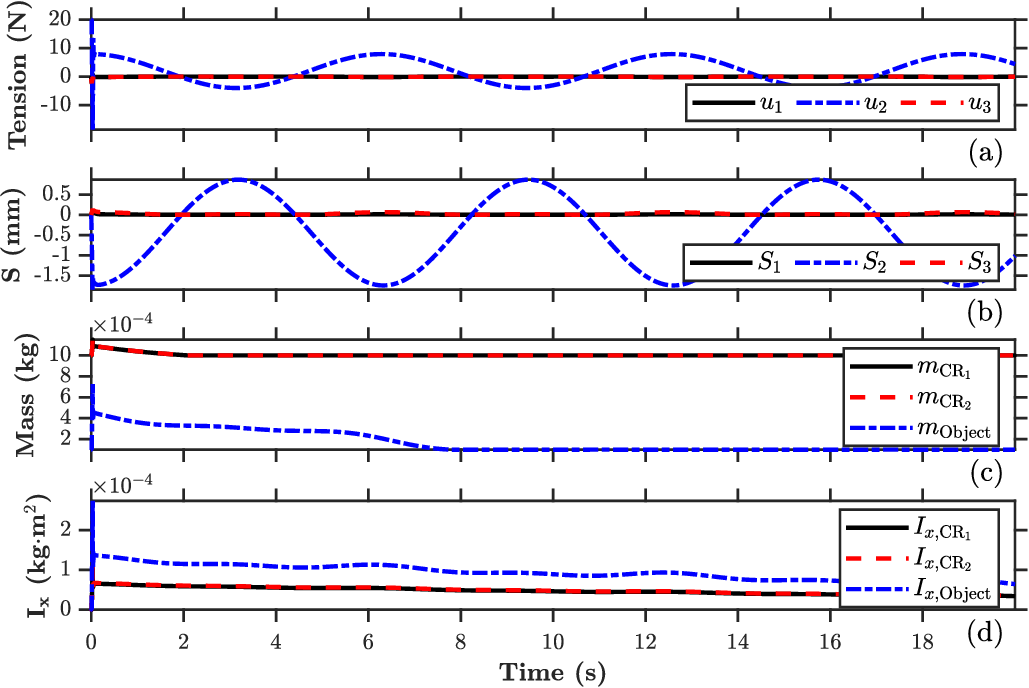}
\caption{Control inputs and adaptive parameter estimates during task-space tracking in simulation. 
(a) tendon tensions $u_1$, $u_2$, and $u_3$; 
(b) tendon displacements $S_1$, $S_2$, and $S_3$; 
(c) online estimates of the mass parameters; 
(d) online estimates of the inertia parameter $I_x$.}
  \label{fig:Sim_effort_adaptation}
\end{figure}

The controller gains and computational performance for both
regulation and tracking are summarized in
Table{\small ~\ref{tab:controller_params}}. The GVS dynamic model was first
implemented in \textsc{MATLAB} and subsequently converted to a MEX
function to improve computational efficiency. The time integration was
performed using a fixed-step explicit Euler method of the form
{\small $\mathbf{y}_{k+1} = \mathbf{y}_k + \Delta t\,\mathbf{f}(\mathbf{y}_k)$},
where the system derivatives were evaluated through the compiled MEX
routine at each time step. A fixed step size was selected to balance
model accuracy and real-time feasibility. Specifically, the sampling
interval {\small $\Delta t$} was chosen to ensure a fast response
while maintaining numerical stability and sufficient accuracy of the
GVS dynamics. Because the measured runtime remains below the integration interval,
the implementation operates in real time.

\begin{table}[t]
\centering
\caption{Simulation gains and runtime.}
\label{tab:controller_params}
\setlength{\tabcolsep}{2pt}
\renewcommand{\arraystretch}{0.85}
\footnotesize
\begin{tabular}{lcc}
\toprule
 & Regulation& Tracking \\
\midrule

$\boldsymbol{\Lambda}$ 
& $1.5{\times}10^{5}\mathrm{diag}(20,1,5)$ 
& $10^{5}\mathrm{diag}(1,1,1)$ \\

$\mathbf{K}_s$ 
& $0.07\mathbf{I}_{16}$
& $0.07\mathbf{I}_{16}$ \\

$\boldsymbol{\Gamma}$ 
& $10^{-4}\mathbf{I}_{12}$
& $10^{-4}\mathbf{I}_{12}$ \\

$\Delta t$ 
& $3.5{\times}10^{-4}$
& $4{\times}10^{-4}$ \\

Avg. run. 
& $0.325\,\mathrm{ms}\,(3.1\,\mathrm{kHz})$ 
& $0.312\,\mathrm{ms}\,(3.2\,\mathrm{kHz})$ \\

RTF
& $1.06$ 
& $1.28$ \\

\bottomrule
\end{tabular}
\end{table}

\vspace{-1.5 em}
\subsection{Comparison with Non-Adaptive and Baseline Controllers}

To provide a fair and meaningful assessment of the proposed projected adaptive controller, its performance is compared with two additional control strategies under identical simulation conditions.

First, a non-adaptive model-based controller is considered. In this case, the same control structure as in {\small \eqref{eq:u_adapt}} is retained, but the adaptation mechanism {\small \eqref{eq:theta_dot}} is disabled by fixing the parameter estimates to their nominal values, {\small $\hat{\boldsymbol{\theta}} = \boldsymbol{\theta}_{\mathrm{nom}}$}. The plant dynamics include a {\small $30\%$} parametric uncertainty in the inertia parameters, while the controller continues to use nominal values. This case evaluates the sensitivity of a model-based controller to structured modeling errors in the absence of online parameter adaptation.

As an additional benchmark, a baseline projected feedback controller is considered. This controller retains the same structure as the control input in {\small \eqref{eq:u_adapt}}, but without the adaptation term {\small $\mathbf{P}\mathbf{Y}(\cdot)\hat{\boldsymbol{\theta}}$}. In other words, the parameter estimation mechanism is disabled and no dynamic compensation is included. The controller therefore relies purely on the projected sliding feedback terms, and serves as a reference for evaluating the benefits of online parameter adaptation.

For all three controllers, the control gains {\small $(\boldsymbol{\Lambda}, \mathbf{K}_s, \boldsymbol{\Gamma})$} are kept identical and are not retuned or adjusted between simulations. 
The gain values are selected according to Table{\small ~\ref{tab:controller_params}}, as described previously. 
This ensures that any performance differences arise solely from the presence or absence of online parameter adaptation, rather than from gain reconfiguration. 
Maintaining identical gain values across all cases provides a consistent and unbiased comparison framework.

\subsubsection{Comparison of Regulation Performance}

A comparison of the regulation performance under three control strategies is presented in Fig.~\ref{fig:Sim_regulation_U_S_3method}. As shown in Fig.~\ref{fig:Sim_regulation_U_S_3method}(a), all controllers achieve regulation, but their transient and steady-state behaviors differ. The adaptive controller yields the smallest steady-state error, followed by the baseline controller, whereas the nominal controller exhibits the largest residual oscillations, as further highlighted by the zoomed view. The corresponding tendon tensions and tendon displacements are shown in Fig.~\ref{fig:Sim_regulation_U_S_3method}(b)--(g). The nominal controller exhibits pronounced oscillations in both the tendon tensions and tendon displacements, while the adaptive and baseline controllers produce smoother responses. This behavior indicates that online parameter adaptation effectively compensates for structured modeling uncertainty, reduces oscillatory corrective action, and improves regulation accuracy.

\begin{figure}
  \centering
  \includegraphics[width=0.48\textwidth]{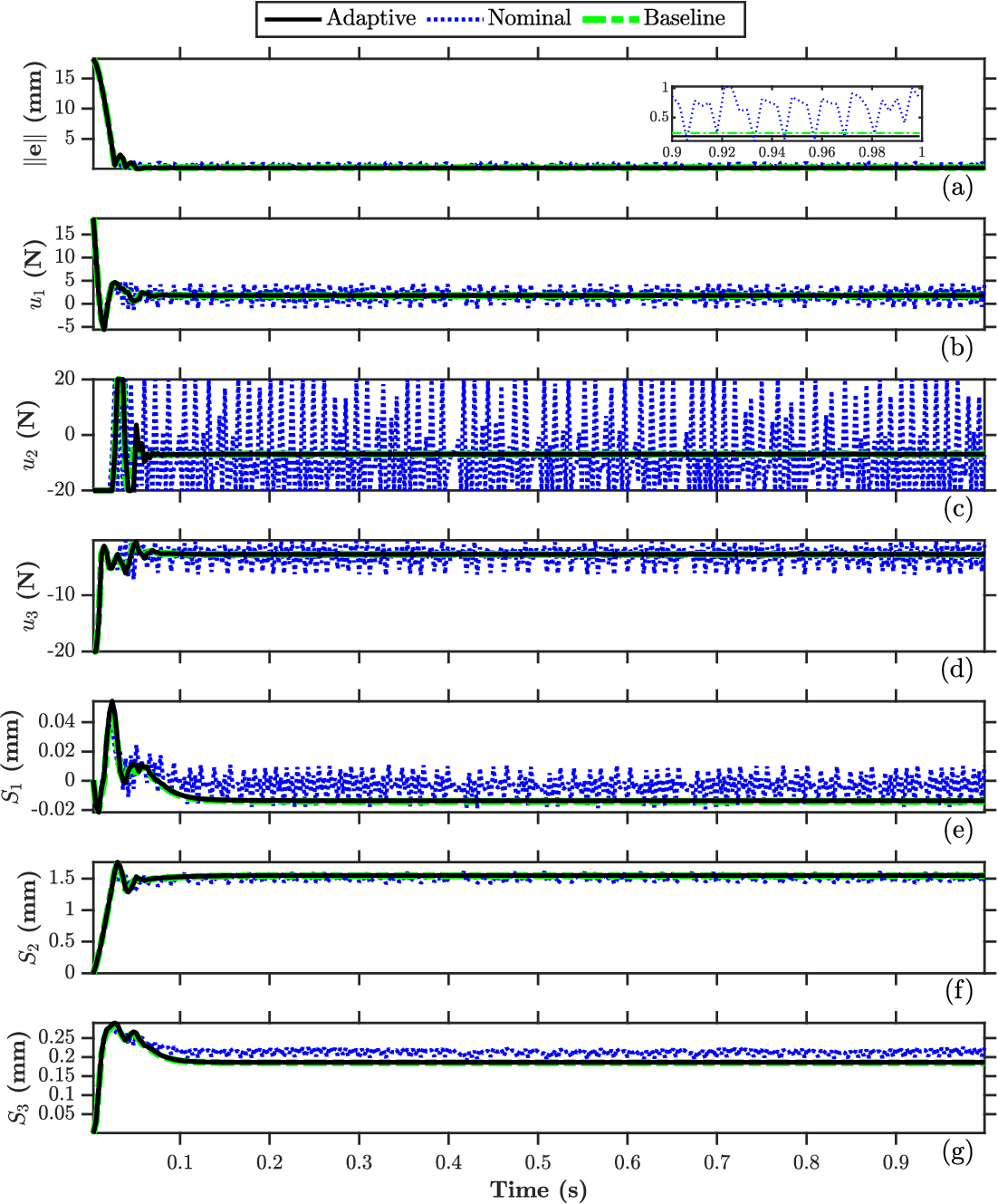}
\caption{Simulation results for regulation:
(a) error norm; 
(b)--(d) tendon tensions $u_1$--$u_3$; 
(e)--(g) tendon displacements $S_1$--$S_3$. 
Comparison of the adaptive, nominal, and baseline controllers.}

  \label{fig:Sim_regulation_U_S_3method}
\end{figure}

\subsubsection{Comparison of Tracking Performance}
The tracking performance under the three control strategies is illustrated in Fig.~\ref{fig:Sim_U_S_3method}. As shown in Fig.~\ref{fig:Sim_U_S_3method}(a), the adaptive controller follows the desired trajectory with high accuracy and smooth convergence, whereas the baseline and nominal controllers exhibit larger tracking fluctuations during motion. The zoomed view of the tracking error over the interval {\small $19$--$20$~s} further highlights these differences, showing that the adaptive controller achieves the lowest steady-state error, followed by the baseline controller, while the nominal controller exhibits the largest error. The corresponding tendon tensions and tendon displacements are shown in Fig.~\ref{fig:Sim_U_S_3method}(b)--(g). Unlike the regulation case, both the nominal and baseline controllers exhibit pronounced oscillatory behavior in the tendon tensions and tendon displacements during trajectory tracking. In contrast, the adaptive controller produces smoother actuation profiles without oscillatory behavior or repeated approaches to the saturation bounds. This demonstrates that online parameter adaptation improves dynamic compensation, reduces oscillatory corrective action, and enhances tracking performance.

\begin{figure}
  \centering
  \includegraphics[width=0.48\textwidth]{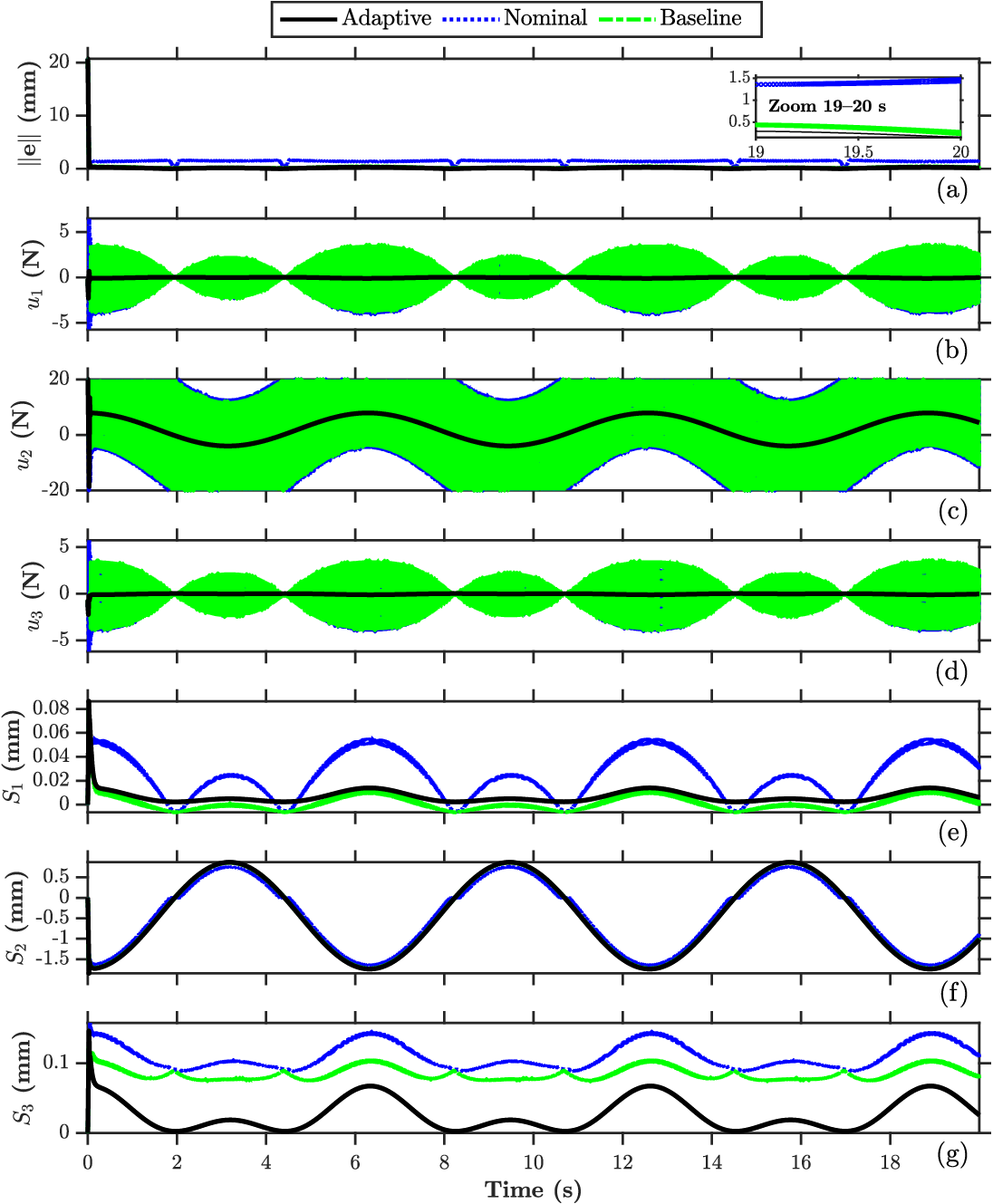}
  \caption{Simulation results for trajectory tracking. 
(a) tracking error norm; 
(b)--(d) tendon tensions $u_1$--$u_3$; 
(e)--(g) tendon displacements $S_1$--$S_3$. 
Comparison of the adaptive, nominal, and baseline controllers.}

  \label{fig:Sim_U_S_3method}
\end{figure}
\vspace{-0.5 em}
\section{Experiment Results}
\label{sec:exp_results}
This section presents the experimental setup and validation of the proposed projected adaptive control framework.
\vspace{-1 em}
\subsection{Experimental Setup}

The mechanical architecture of the two-arm tendon-driven CCR platform is shown in Fig.~\ref{fig:exp_setup}. Each continuum arm consists of a spring steel backbone with circular cross-section. Four Kevlar tendons are routed through uniformly spaced spacer disks with radial offset {\small $r_t=1.8~\mathrm{mm}$} and angular separation of {\small $90^\circ$}, forming two orthogonal tendon pairs that enable spatial bending. The two arms are mounted on rigid base plates separated by {\small $629~\mathrm{mm}$}. The manipulated flexible object is a Nitinol rod that rigidly attached to the distal ends of the two arms, forming a closed-chain configuration.
Each tendon is actuated by a Dynamixel AX-12A servo motor (Robotis, Seoul, Korea) 
through serial communication. The tendon displacement
computed by the dynamic model is mapped to motor angular
position using the known actuator shaft diameter. 
In both simulation and experimental studies, three independent tendon inputs are actively controlled, consistent with the task-space position regulation and trajectory tracking objectives. Specifically, {\small $u_1$} and {\small $u_2$} correspond to the two orthogonal tendon pairs of CR{\small1}, generating bending in the {\small $xz$}- and {\small $yz$}-planes, respectively (see Fig.~\ref{fig:exp_setup}). The third input, {\small $u_3$}, actuates the corresponding tendon pair of CR{\small2} in the {\small $xz$}-plane. 

Three-dimensional tip position measurements are obtained using a Tracker 3.0 Vicon motion capture system (Bilston, UK). Four reflective markers are attached to the distal tip of each arm (eight in total) for tip reconstruction, as shown in Fig.~\ref{fig:exp_setup}, while three additional markers mounted on the base plate define the global reference frame. The marker positions are processed in real time to reconstruct the tip locations, from which the midpoint of the manipulated object is computed and used as the feedback signal.

The measured midpoint position is directly fed into a compiled MEX function that encapsulates the complete system dynamics, including the CRs model, the closed-chain coupling terms, and the control law. This unified dynamic function is executed within a fixed-step explicit Euler integration loop. The MEX function then evaluates the coupled dynamics and control law to generate the updated generalized coordinates {\small $\mathbf{q}$} and generalized velocities {\small $\dot{\mathbf{q}}$} for the next integration step.
The generalized coordinates {\small $\mathbf{q}$} are provided to the actuation mapping function to compute the corresponding tendon displacements. These tendon displacements are then converted into motor position commands based on the known actuator shaft diameter and transmitted to the motors.

\begin{figure}
\centering
\includegraphics[width=0.9\columnwidth,trim=1.5cm 1cm 3cm 1.5cm,clip]{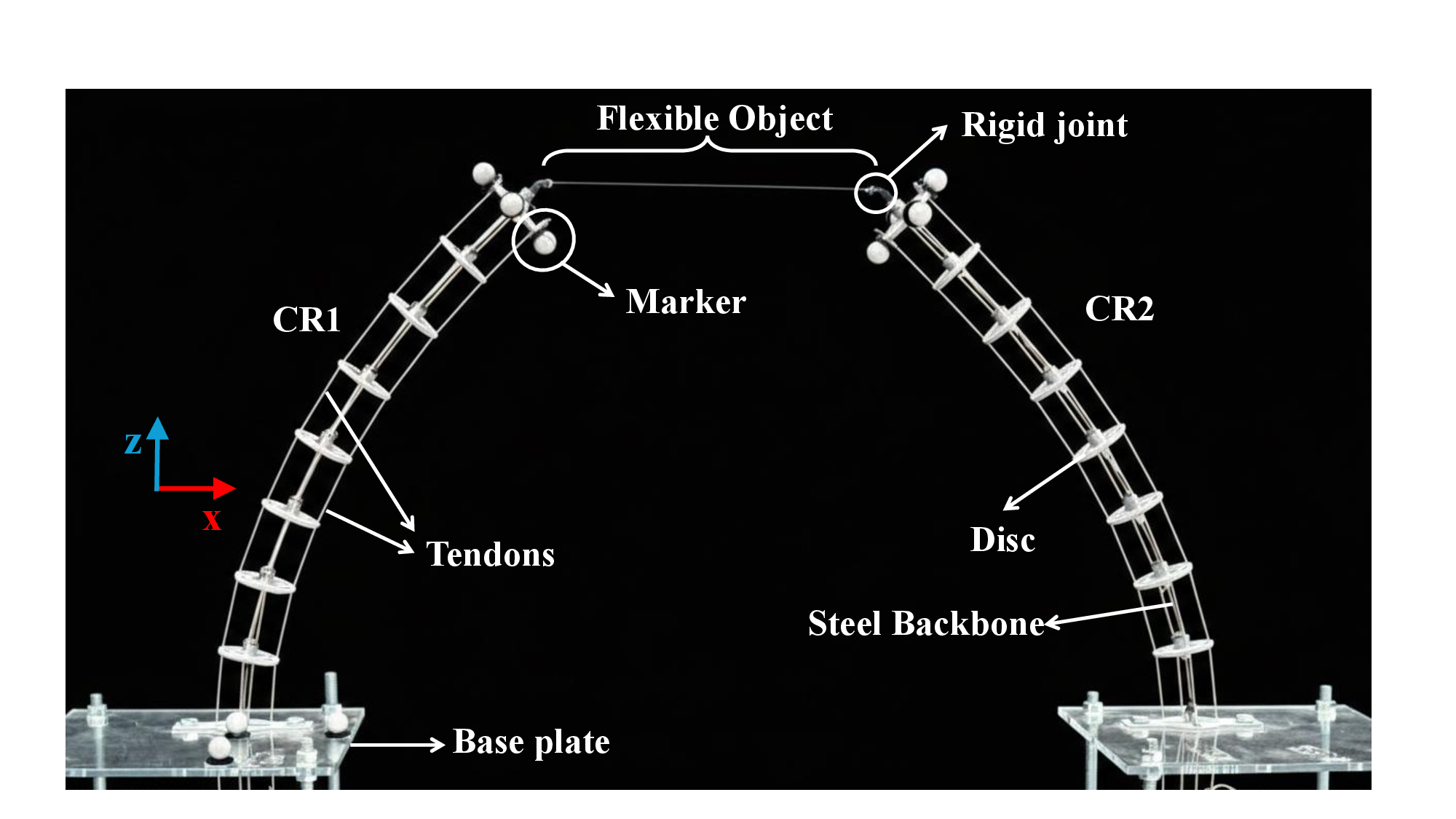}
\caption{Experimental platform of the two-arm tendon-driven CCR system, showing the CRs, flexible object, tendon routing, and markers.}
\label{fig:exp_setup}
\end{figure}

The sensing stage and motor actuation stage operate at {\small $25~\mathrm{Hz}$}, providing synchronized Vicon measurements and motor updates. The entire sensing–computation–actuation process is executed in real time. The computational time per iteration remains below the sampling interval, ensuring stable real-time implementation without buffer accumulation or timing overruns. The controller gains are kept identical to those used in simulation, as listed in Table{\small ~\ref{tab:controller_params}}, to ensure a consistent comparison between simulation and experimental results.

The experimental study follows the same structure as the simulation analysis. First, task-space regulation is performed, followed by trajectory tracking. In both cases, the three control strategies (Adaptive, Nominal, and Baseline) are compared. 
\vspace{-1.5 em}
\subsection{Task-Space Regulation}

The experimental regulation results for the adaptive controller are presented in Fig.~{\small \ref{fig:EXP_regulation_performance}}. Each experiment is repeated five times to evaluate repeatability. The individual trials are illustrated in gray, while the bold black curve represents the mean response across the five runs.
The results demonstrate consistent convergence to the desired position with small steady-state error. The error norm rapidly decreases and remains close to zero after the transient phase, indicating stable real-time implementation.

\begin{figure}
  \centering
  \includegraphics[width=0.48\textwidth]{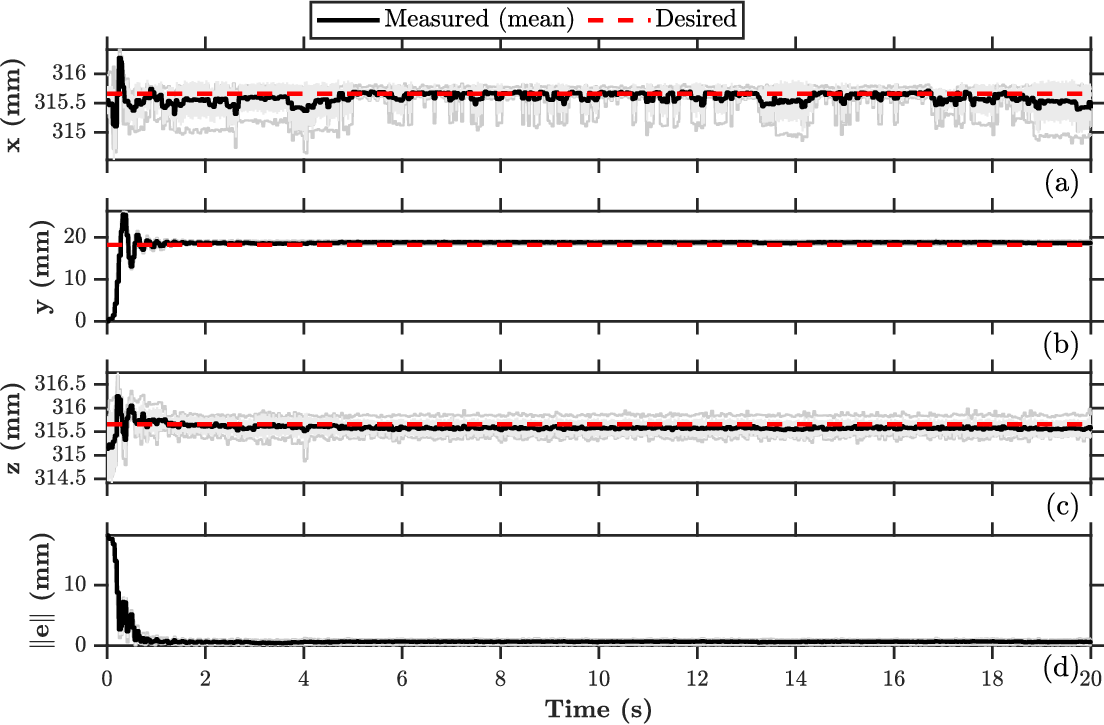}
 \caption{Task-space regulation of the manipulated object midpoint over five trials. 
(a)--(c) $x$-, $y$-, and $z$-coordinates of the midpoint, and (d) error norm. 
Gray curves denote individual trials; the bold curve shows mean response.}
  \label{fig:EXP_regulation_performance}
\end{figure}

\vspace{-1 em}
\subsection{Task-Space Trajectory Tracking}

The experimental tracking performance of the adaptive controller is presented in Fig.~{\small \ref{fig:EXP_tracking_performance}}. The results demonstrate accurate trajectory tracking under real-time implementation. The measured signals closely follow the desired trajectory, and the tracking error remains bounded with small steady-state magnitude. Minor variations between repetitions are attributable to measurement noise, actuator nonlinearities, and unmodeled dynamics inherent to the physical system.

\begin{figure}
  \centering
  \includegraphics[width=0.48\textwidth]{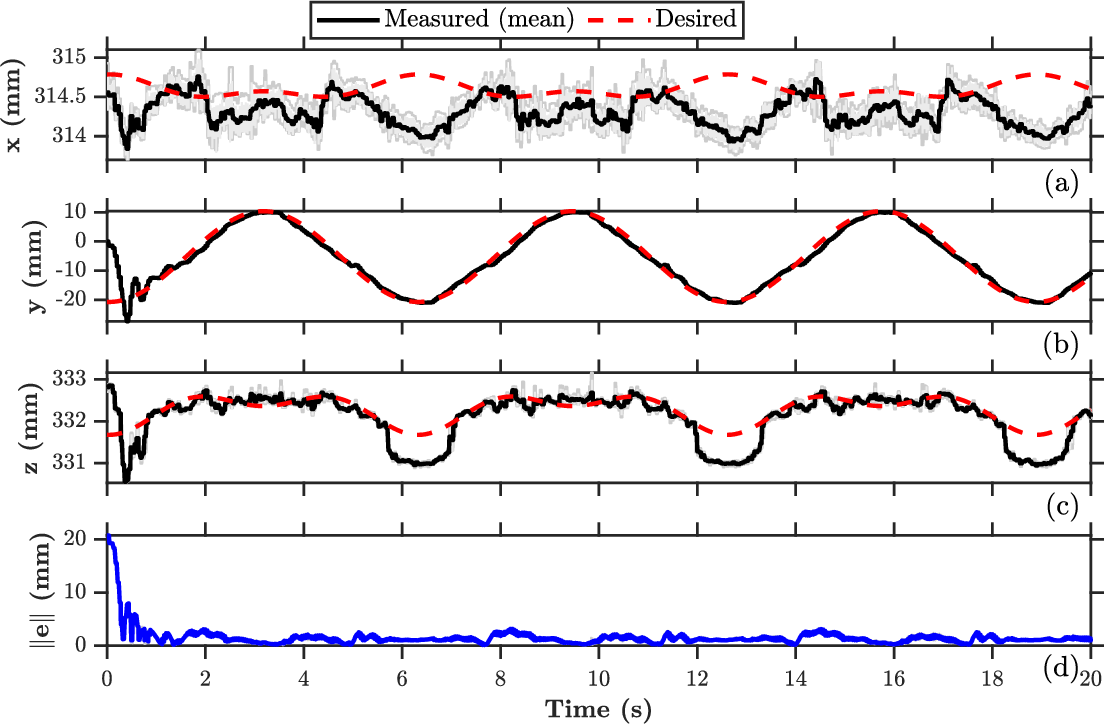}
\caption{Task-space tracking of the manipulated object midpoint over five trials. 
(a)--(c) $x$-, $y$-, and $z$-coordinates of the midpoint, and (d) tracking error. 
Gray curves denote individual trials; the bold curve shows the mean response.}
  \label{fig:EXP_tracking_performance}
\end{figure}

\vspace{-0.5 em}
\subsection{Experimental Comparison of Control Strategies}

A comparative experimental evaluation is performed to assess the performance of the three control strategies under real-time implementation.

\subsubsection{Experimental Regulation} 
The regulation results are presented in Fig.~\ref{fig:EXP_regulation_3comparison}. As shown in Fig.~\ref{fig:EXP_regulation_3comparison}(a)--(c), both the adaptive and baseline controllers converge smoothly to the desired midpoint position, whereas the nominal controller exhibits noticeable fluctuations around the reference. The error norm in Fig.~\ref{fig:EXP_regulation_3comparison}(d), together with the zoomed view over {\small $19$--$20$~s}, further highlights these differences, showing that the adaptive controller achieves the smallest steady-state error, followed by the baseline controller, while the nominal controller yields the largest error. The norm of tendon displacement in Fig.~\ref{fig:EXP_regulation_3comparison}(e) shows a similar trend, with smoother behavior for the adaptive and baseline controllers and more oscillatory variations for the nominal controller. These results confirm that online parameter adaptation improves regulation accuracy under experimental conditions.

\begin{figure}
  \centering
  \includegraphics[width=0.48\textwidth]{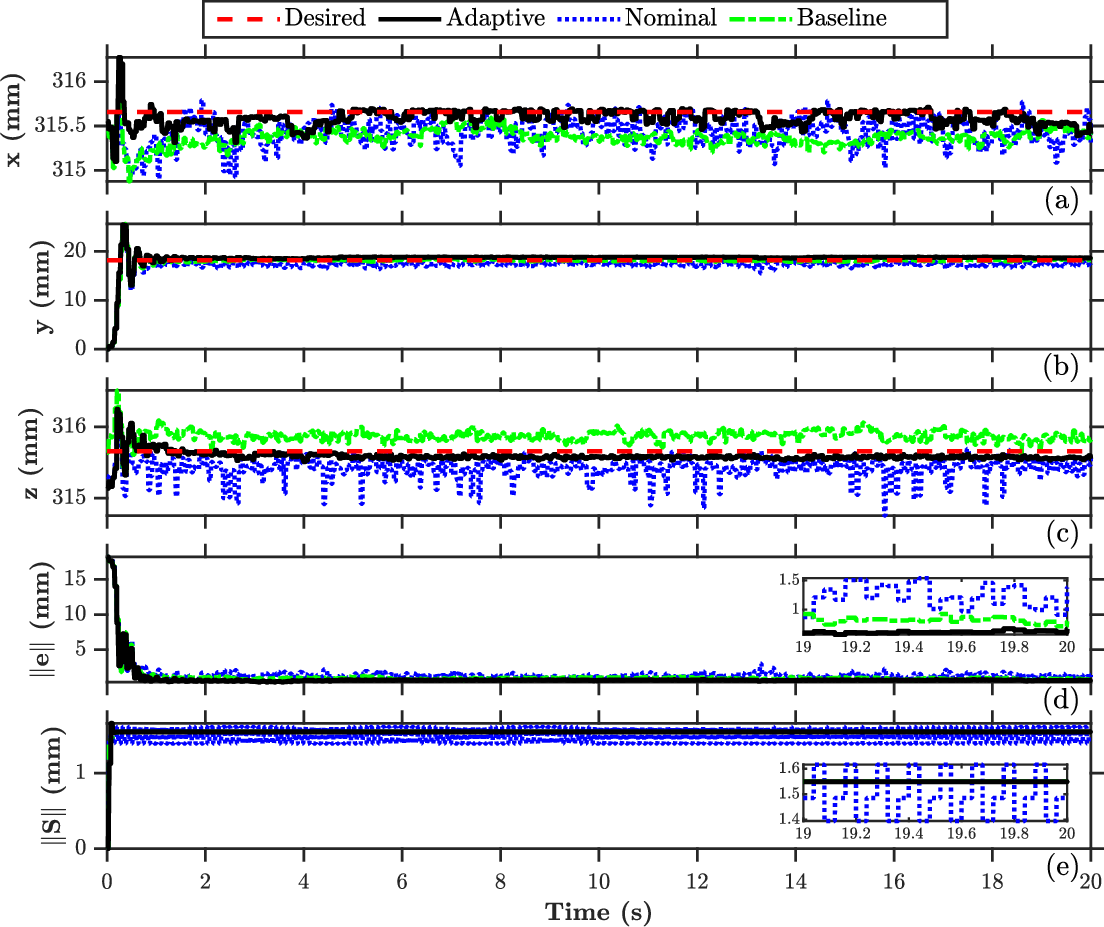}
  \caption{Experimental regulation comparison under three controllers.
(a) $x$, (b) $y$, (c) $z$, (d) tracking error $\|\mathbf{e}\|$, and (e) tendon displacement norm $\|\mathbf{S}\|$.
Insets in (d) and (e) show zoomed views over $19$--$20\,\mathrm{s}$.}
  \label{fig:EXP_regulation_3comparison}
\end{figure}

\subsubsection{Experimental Tracking}
The tracking results are presented in Fig.~\ref{fig:EXP_tracking_3comparison}. As shown in Fig.~\ref{fig:EXP_tracking_3comparison}(a)--(c), the adaptive controller provides the highest tracking accuracy, while the nominal and baseline controllers exhibit larger fluctuations around the desired trajectory. The tracking error in Fig.~\ref{fig:EXP_tracking_3comparison}(d) further confirms that the adaptive controller achieves the smallest error during motion. The norm of tendon displacement in Fig.~\ref{fig:EXP_tracking_3comparison}(e) shows smoother behavior for the adaptive controller, whereas the nominal and baseline controllers.

\begin{figure}
  \centering
  \includegraphics[width=0.48\textwidth]{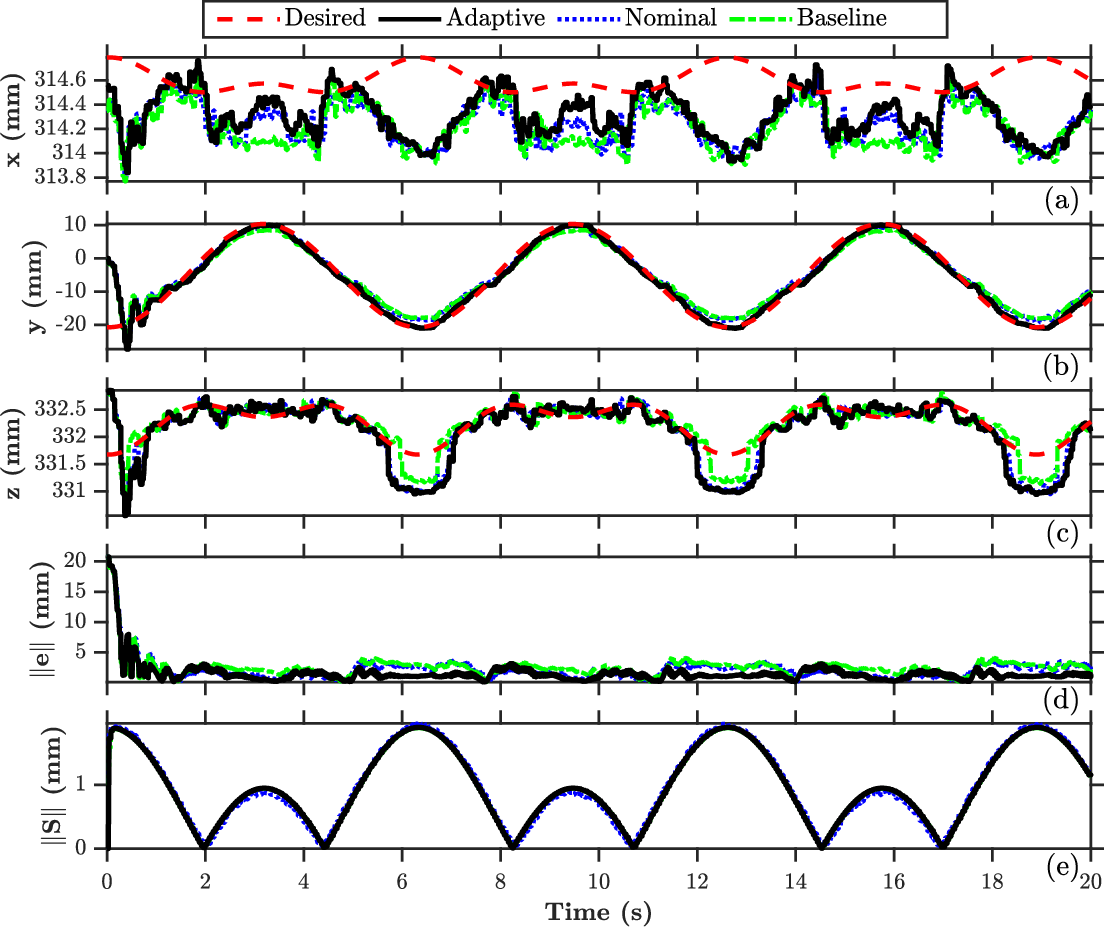}
 \caption{Experimental tracking comparison under three controllers.
(a) $x$, (b) $y$, (c) $z$, (d) tracking error $\|\mathbf{e}\|$, and (e) tendon displacement norm $\|\mathbf{S}\|$.}
  \label{fig:EXP_tracking_3comparison}
\end{figure}

Table~\ref{tab:performance_comparison} summarizes the quantitative performance of the three controllers in simulation and experimental studies for task-space regulation and trajectory tracking. Four metrics are used for evaluation.
The root-mean-square error (RMSE) of the task-space error is used to quantify the average tracking performance over time. The total variation of the error, which characterizes the oscillatory or zig-zag behavior of the error signal, is computed as
{\small
\begin{align*}
    \mathrm{TV}_e 
= \sum_{k=1}^{N-1} 
\left| \|\mathbf{e}(t_{k+1})\| - \|\mathbf{e}(t_k)\| \right|,
\end{align*}
}

\noindent where {\small $\mathbf{e}(t_k)$} denotes the task-space position error at time {\small $t_k$}, and {\small $N$} is the total number of samples.
Similarly, the total variation of tendon displacement is defined using the Euclidean norm of the displacement vector {\small $\mathbf{S} = [S_1 \; S_2 \; S_3]^{\mathsf{T}}$}, and is used to assess actuation smoothness.

In simulation, the number of saturation events is defined as the total number of transitions for which any tendon tension reaches its imposed bound, indicating actuator stress. For the experimental results, the reported values correspond to the mean over five repeated trials. As shown in {\small Table~\ref{tab:performance_comparison}}, the adaptive controller consistently achieves lower RMSE, reduced oscillatory behavior (smaller {\small $\mathrm{TV}_e$} and {\small $\mathrm{TV}_S$}), and minimal saturation events, demonstrating improved accuracy, robustness and smoother actuation under both regulation and tracking conditions.

\begin{table}
\centering
\caption{Performance comparison in simulation and experiment.}
\label{tab:performance_comparison}
\footnotesize
\setlength{\tabcolsep}{2pt}
\renewcommand{\arraystretch}{1}
\resizebox{\columnwidth}{!}{%
\begin{tabular}{l l l c c c c}
\toprule
Domain & Task & Ctrl. & RMSE [mm] & TV$_e$ & TV$_S$ & Sat. \\
\midrule
\multirow{6}{*}{Sim.}
& \multirow{3}{*}{Reg.}
& Adp.  & \textbf{2.03} & \textbf{22.61} & \textbf{2.46} & 2 \\
& & Base. & 2.04 & 23.28 & 2.64 & 2 \\
& & Nom.  & 2.13 & 125.60 & 25.69 & 84 \\
\cmidrule(lr){2-7}
& \multirow{3}{*}{Track.}
& Adp.  & \textbf{0.54} & \textbf{25.70} & \textbf{18.64} & \textbf{0} \\
& & Base. & 0.58 & 7923.29 & 1797.05 & 32967 \\
& & Nom.  & 1.52 & 9411.96 & 1707.03 & 30751 \\
\midrule
\multirow{6}{*}{Exp.}
& \multirow{3}{*}{Reg.}
& Adp.  & \textbf{1.98 $\pm$ 0.12} & \textbf{72.99 $\pm$ 8.15} & \textbf{1.55 $\pm$ 0.02} & -- \\
& & Base. & 2.10 $\pm$ 0.13 & 91.53 $\pm$ 16.94 & 1.80 $\pm$ 0.03 & -- \\
& & Nom.  & 2.33 $\pm$ 0.18 & 265.33 $\pm$ 38.89 & 74.73 $\pm$ 3.17 & -- \\
\cmidrule(lr){2-7}
& \multirow{3}{*}{Track.}
& Adp.  & \textbf{2.50 $\pm$ 0.02} & \textbf{384.17 $\pm$ 4.38} & \textbf{19.36 $\pm$ 1.05} & -- \\
& & Base. & 3.18 $\pm$ 0.02 & 398.87 $\pm$ 5.25 & 20.51 $\pm$ 1.76 & -- \\
& & Nom.  & 2.98 $\pm$ 0.03 & 407.33 $\pm$ 10.80 & 21.42 $\pm$ 1.82 & -- \\
\bottomrule
\end{tabular}%
}
\end{table}

\section{Conclusion}
This paper presented a projected adaptive control framework for CCRs coupled through a closed kinematic chain. The CCR system was modeled using the GVS formulation, and the closed-chain interaction was represented through Pfaffian velocity constraints. By projecting the dynamics onto the null space of the constraint Jacobian, the closed-chain model was expressed in the constraint-consistent motion subspace, providing a suitable foundation for controller design. Based on this formulation, an adaptive control law with parameter-linear regressors was developed to compensate online for uncertain mass and inertial properties of the CRs and the manipulated flexible object. Lyapunov analysis established bounded closed-loop signals and convergence of the sliding variable to zero, which implies asymptotic convergence of the task-space tracking error. Simulation and experimental results in both regulation and trajectory tracking demonstrated that the proposed controller achieves improved accuracy, reduced oscillatory behavior, smoother actuation, and fewer saturation events compared with the nominal and baseline controllers. Overall, the proposed framework provides a systematic approach for dynamic modeling and adaptive control of constrained CCR systems and supports stable real-time co-manipulation of a flexible object. Future work will investigate adaptive position/force control in the presence of contact interactions and learning-enhanced adaptive schemes for handling more complex uncertainties.

\bibliographystyle{IEEEtran}
\bibliography{references}

@book{murray2017mathematical,
  title={A Mathematical Introduction to Robotic Manipulation},
  author={Murray, Richard M and Li, Zexiang and Sastry, S Shankar},
  year={2017},
  publisher={CRC press}
}

@article{renda2020geometric,
  title={A geometric variable-strain approach for static modeling of soft manipulators with tendon and fluidic actuation},
  author={Renda, Federico and Armanini, Costanza and Lebastard, Vincent and Candelier, Fabien and Boyer, Frederic},
  journal={IEEE Robotics and Automation Letters},
  volume={5},
  number={3},
  pages={4006--4013},
  year={2020},
  publisher={IEEE}
}

@article{renda2018discrete,
  title={Discrete cosserat approach for multisection soft manipulator dynamics},
  author={Renda, Federico and Boyer, Fr{\'e}d{\'e}ric and Dias, Jorge and Seneviratne, Lakmal},
  journal={IEEE Transactions on Robotics},
  volume={34},
  number={6},
  pages={1518--1533},
  year={2018},
  publisher={IEEE}
}

@article{armanini2021discrete,
  title={Discrete cosserat approach for closed-chain soft robots: Application to the fin-ray finger},
  author={Armanini, Costanza and Hussain, Irfan and Iqbal, Muhammad Zubair and Gan, Dongming and Prattichizzo, Domenico and Renda, Federico},
  journal={IEEE Transactions on Robotics},
  volume={37},
  number={6},
  pages={2083--2098},
  year={2021},
  publisher={IEEE}
}

@article{boyer2020dynamics,
  title={Dynamics of continuum and soft robots: A strain parameterization based approach},
  author={Boyer, Frederic and Lebastard, Vincent and Candelier, Fabien and Renda, Federico},
  journal={IEEE Transactions on Robotics},
  volume={37},
  number={3},
  pages={847--863},
  year={2020},
  publisher={IEEE}
}

@article{renda2014dynamic,
  title={Dynamic model of a multibending soft robot arm driven by cables},
  author={Renda, Federico and Giorelli, Michele and Calisti, Marcello and Cianchetti, Matteo and Laschi, Cecilia},
  journal={IEEE Transactions on Robotics},
  volume={30},
  number={5},
  pages={1109--1122},
  year={2014},
  publisher={IEEE}
}

@article{hussain2021compliant,
  title={Compliant gripper design, prototyping, and modeling using screw theory formulation},
  author={Hussain, Irfan and Malvezzi, Monica and Gan, Dongming and Iqbal, Zubair and Seneviratne, Lakmal and Prattichizzo, Domenico and Renda, Federico},
  journal={The International Journal of Robotics Research},
  volume={40},
  number={1},
  pages={55--71},
  year={2021},
  publisher={SAGE Publications Sage UK: London, England}
}

@article{aghili2011projection,
  title={Projection-based control of parallel mechanisms},
  author={Aghili, Farhad},
 journal = {Journal of Computational and Nonlinear Dynamics},
volume = {6},
number = {3},
pages = {031009},
year = {2011}
  
}

@article{laschi2016lessons,
  title={Lessons from animals and plants: The symbiosis of morphological computation and soft robotics},
  author={Laschi, Cecilia and Mazzolai, Barbara},
  journal={IEEE Robotics \& Automation Magazine},
  volume={23},
  number={3},
  pages={107--114},
  year={2016},
  publisher={IEEE}
}

@article{simaan2018medical,
  title={Medical technologies and challenges of robot-assisted minimally invasive intervention and diagnostics},
  author={Simaan, Nabil and Yasin, Rashid M and Wang, Long},
  journal={Annual Review of Control, Robotics, and Autonomous Systems},
  volume={1},
  number={1},
  pages={465--490},
  year={2018},
  publisher={Annual Reviews}
}

@article{lotfavar2017cooperative,
  title={Cooperative continuum robots: Concept, modeling, and workspace analysis},
  author={Lotfavar, Amir and Hasanzadeh, Shahir and Janabi-Sharifi, Farrokh},
  journal={IEEE Robotics and Automation Letters},
  volume={3},
  number={1},
  pages={426--433},
  year={2017},
  publisher={IEEE}
}

@article{li2025collaborative,
  title={Collaborative Continuum Robots: A Survey},
  author={Li, Xinyu and Tang, Qian and Yin, Guoxin and Zheng, Gang and Burgner-Kahrs, Jessica and Stefanini, Cesare and Wu, Ke},
  journal={arXiv preprint arXiv:2601.10721},
  year={2025}
}

@article{chikhaoui2018toward,
  title={Toward motion coordination control and design optimization for dual-arm concentric tube continuum robots},
  author={Chikhaoui, Mohamed Taha and Granna, Josephine and Starke, Julia and Burgner-Kahrs, Jessica},
  journal={IEEE Robotics and Automation Letters},
  volume={3},
  number={3},
  pages={1793--1800},
  year={2018},
  publisher={IEEE}
}

@article{peng2023modeling,
  title={Modeling, cooperative planning and compliant control of multi-arm space continuous robot for target manipulation},
  author={Peng, Jianqing and Wu, Haoxuan and Zhang, Chi and Chen, Qihan and Meng, Deshan and Wang, Xueqian},
  journal={Applied Mathematical Modelling},
  volume={121},
  pages={690--713},
  year={2023},
  publisher={Elsevier}
}

@article{jalali2021dynamic,
  title={Dynamic modeling of tendon-driven co-manipulative continuum robots},
  author={Jalali, Amir and Janabi-Sharifi, Farrokh},
  journal={IEEE Robotics and Automation Letters},
  volume={7},
  number={2},
  pages={1643--1650},
  year={2021},
  publisher={IEEE}
}

@article{norouzi2021switching,
  title={A switching image-based visual servoing method for cooperative continuum robots},
  author={Norouzi-Ghazbi, Somayeh and Janabi-Sharifi, Farrokh},
  journal={Journal of Intelligent \& Robotic Systems},
  volume={103},
  number={3},
  pages={42},
  year={2021},
  publisher={Springer}
}

@article{webster2010design,
  title={Design and kinematic modeling of constant curvature continuum robots: A review},
  author={Webster III, Robert J and Jones, Bryan A},
  journal={The International Journal of Robotics Research},
  volume={29},
  number={13},
  pages={1661--1683},
  year={2010},
  publisher={SAGE Publications Sage UK: London, England}
}

@article{cianchetti2014soft,
  title={Soft robotics technologies to address shortcomings in today's minimally invasive surgery: the STIFF-FLOP approach},
  author={Cianchetti, Matteo and Ranzani, Tommaso and Gerboni, Giada and Nanayakkara, Thrishantha and Althoefer, Kaspar and Dasgupta, Prokar and Menciassi, Arianna},
  journal={Soft robotics},
  volume={1},
  number={2},
  pages={122--131},
  year={2014},
  publisher={Mary Ann Liebert, Inc. 140 Huguenot Street, 3rd Floor New Rochelle, NY 10801 USA}
}

@article{burgner2015continuum,
  title={Continuum robots for medical applications: A survey},
  author={Burgner-Kahrs, Jessica and Rucker, D Caleb and Choset, Howie},
  journal={IEEE transactions on robotics},
  volume={31},
  number={6},
  pages={1261--1280},
  year={2015},
  publisher={IEEE}
}

@article{rus2015design,
  title={Design, fabrication and control of soft robots},
  author={Rus, Daniela and Tolley, Michael T},
  journal={Nature},
  volume={521},
  number={7553},
  pages={467--475},
  year={2015},
  publisher={Nature Publishing Group UK London}
}

@article{trivedi2008soft,
  title={Soft robotics: Biological inspiration, state of the art, and future research},
  author={Trivedi, Deepak and Rahn, Christopher D and Kier, William M and Walker, Ian D},
  journal={Applied bionics and biomechanics},
  volume={5},
  number={3},
  pages={99--117},
  year={2008},
  publisher={Taylor \& Francis}
}

@article{DANESH2025105953,
title = {Backstepping control of tendon-driven continuum robots in large deflections using the Cosserat rod model},
author = {Rana Danesh and Farrokh Janabi-Sharifi},
journal = {Mechanism and Machine Theory},
volume = {208},
pages = {105953},
year = {2025},
}

@article{ma2022collaborative,
  title={Collaborative continuum robots for remote engineering operations},
  author={Ma, Nan and Monk, Stephen and Cheneler, David},
  journal={Biomimetics},
  volume={8},
  number={1},
  pages={4},
  year={2022},
  publisher={MDPI}
}

@inproceedings{wang2019design,
  title={Design, control and analysis of a dual-arm continuum flexible robot system},
  author={Wang, Chenxi and Li, Zhi and Ren, Yunfan and Deng, Yuwen and Song, Shuang},
  booktitle={2019 IEEE International Conference on Robotics and Biomimetics (ROBIO)},
  pages={948--953},
  year={2019},
}

@article{cheng2023dexterity,
  title={Dexterity enhancement of continuum robot for natural orifice transluminal endoscopic surgery in the dual-manipulator collaborative space},
  author={Cheng, Tianyu and Zhang, Gang and Sun, Jianjun and Zhang, Tao and Du, Fuxin},
  journal={The International Journal of Medical Robotics and Computer Assisted Surgery},
  volume={19},
  number={4},
  pages={e2516},
  year={2023},
  publisher={Wiley Online Library}
}

@article{dai2023novel,
  title={A novel space robot with triple cable-driven continuum arms for space grasping},
  author={Dai, Yicheng and Li, Zuan and Chen, Xinjie and Wang, Xin and Yuan, Han},
  journal={Micromachines},
  volume={14},
  number={2},
  pages={416},
  year={2023},
  publisher={MDPI}
}

@article{quaicoe2024cooperative,
  title={Cooperative Control of Bimanual Continuum Robots for Automated Knot-Tying in Robot-Assisted Surgical Suturing},
  author={Quaicoe, Enoch and Nada, Ayman and Ishii, Hiroyuki and El-Hussieny, Haitham},
  journal={Journal of Robotics and Control (JRC)},
  volume={5},
  number={4},
  pages={1149--1165},
  year={2024}
}

@article{nuelle2020modeling,
  title={Modeling, calibration, and evaluation of a tendon-actuated planar parallel continuum robot},
  author={Nuelle, Kathrin and Sterneck, Tim and Lilge, Sven and Xiong, Dezhu and Burgner-Kahrs, Jessica and Ortmaier, Tobias},
  journal={IEEE Robotics and Automation Letters},
  volume={5},
  number={4},
  pages={5811--5818},
  year={2020},
  publisher={IEEE}
}

@article{wen2023modeling,
  title={Modeling and analysis of tendon-driven parallel continuum robots under constant curvature and pseudo-rigid-body assumptions},
  author={Wen, Kefei and Burgner-Kahrs, Jessica},
  journal={Journal of Mechanisms and Robotics},
  volume={15},
  number={4},
  pages={041003},
  year={2023},
  publisher={American Society of Mechanical Engineers}
}

@article{wang2024development,
  title={Development of a new cable-driven planar parallel continuum robot using compound kinematic calibration method},
  author={Wang, Zhengyu and Liu, Xuchang and Jia, Zirui and Yu, Xiang and Pei, Zongkun and Yang, Jun},
  journal={Journal of Mechanisms and Robotics},
  volume={16},
  number={10},
  pages={101014},
  year={2024},
  publisher={American Society of Mechanical Engineers}
}

@article{lilge2022kinetostatic,
  title={Kinetostatic modeling of tendon-driven parallel continuum robots},
  author={Lilge, Sven and Burgner-Kahrs, Jessica},
  journal={IEEE Transactions on Robotics},
  volume={39},
  number={2},
  pages={1563--1579},
  year={2022},
  publisher={IEEE}
}

@inproceedings{mahoney2016reconfigurable,
  title={Reconfigurable parallel continuum robots for incisionless surgery},
  author={Mahoney, Arthur W and Anderson, Patrick L and Swaney, Philip J and Maldonado, Fabien and Webster, Robert J},
  booktitle={2016 IEEE/RSJ International Conference on Intelligent Robots and Systems (IROS)},
  pages={4330--4336},
  year={2016},
  
}

@article{wang2026strain,
  title={Strain-based Shape and 3D Force Estimation for Rod-driven Continuum Robots with Stretch Sensors},
  author={Wang, Peiyi and Feliu-Talegon, Daniel and Sun, Yuchen and Xie, Zhexin and Xin, Wenci and Nazeer, Muhammad Sunny and Della Santina, Cosimo and Laschi, Cecilia and Renda, Federico},
  journal={IEEE Transactions on Robotics},
  year={2026},
  publisher={IEEE}
}

@article{wang2021eccentric,
  title={Eccentric tube robots as multiarmed steerable sheaths},
  author={Wang, Jiaole and Peine, Joseph and Dupont, Pierre E},
  journal={IEEE Transactions on Robotics},
  volume={38},
  number={1},
  pages={477--490},
  year={2021},
  publisher={IEEE}
}

@inproceedings{mitros2022design,
  title={Design and quasistatic modelling of hybrid continuum multi-arm robots},
  author={Mitros, Zisos and Sadati, SM Hadi and Nousias, Sotirios and Da Cruz, Lyndon and Bergeles, Christos},
  booktitle={2022 International Conference on Robotics and Automation (ICRA)},
  pages={9607--9613},
  year={2022},
  
}

@article{mathew2022sorosim,
  title={Sorosim: A matlab toolbox for hybrid rigid--soft robots based on the geometric variable-strain approach},
  author={Mathew, Anup Teejo and Hmida, Ikhlas Ben and Armanini, Costanza and Boyer, Frederic and Renda, Federico},
  journal={IEEE Robotics \& Automation Magazine},
  volume={30},
  number={3},
  pages={106--122},
  year={2022},
  publisher={IEEE}
}

@article{yu2023model,
  title={Model-free synchronous motion generation of multiple heterogeneous continuum robots},
  author={Yu, Peng and Tan, Ning and Wu, Yuyang and Qiu, Binbin and Huang, Kai},
  journal={IEEE Transactions on Industrial Informatics},
  volume={20},
  number={3},
  pages={3209--3221},
  year={2023},
  publisher={IEEE}
}

@inproceedings{tan2021synchronous,
  title={Synchronous motion generation of multiple continuum robots based on a Jacobian-estimation strategy},
  author={Tan, Ning and Hu, Ruikun and Wu, Yuyang and Zhang, Xinyu and Ni, Fenglei and Sun, Zhenglong},
  booktitle={2021 IEEE International Conference on Robotics and Biomimetics (ROBIO)},
  pages={475--482},
  year={2021},
  
}

@inproceedings{zhang2022cooperative,
  title={Cooperative control of dual-arm concentric tube continuum robots},
  author={Zhang, Hanna Jiamei and Lilge, Sven and Chikhaoui, M Taha and Burgner-Kahrs, Jessica},
  booktitle={2022 International Conference on Manipulation, Automation and Robotics at Small Scales (MARSS)},
  pages={1--6},
  year={2022},
  
}

@article{sabetian2019self,
  title={Self-collision detection and avoidance for dual-arm concentric tube robots},
  author={Sabetian, Saba and Looi, Thomas and Diller, Eric D and Drake, James},
  journal={IEEE Robotics and Automation Letters},
  year={2019},
  volume={},
  number={},
  pages={1-1},
  publisher={IEEE}
}

@inproceedings{ji2022omnidirectional,
  title={Omnidirectional walking of a quadruped robot enabled by compressible tendon-driven soft actuators},
  author={Ji, Qinglei and Fu, Shuo and Feng, Lei and Andrikopoulos, George and Wang, Xi Vincent and Wang, Lihui},
  booktitle={2022 IEEE/RSJ International Conference on Intelligent Robots and Systems (IROS)},
  pages={11015--11022},
  year={2022},
  
}

@inproceedings{de2013introducing,
  title={Introducing STRAS: A new flexible robotic system for minimally invasive surgery},
  author={De Donno, Antonio and Zorn, Lucile and Zanne, Philippe and Nageotte, Florent and de Mathelin, Michel},
  booktitle={2013 IEEE International Conference on Robotics and Automation},
  pages={1213--1220},
  year={2013},
  
}

@article{hwang2020k,
  title={K-FLEX: a flexible robotic platform for scar-free endoscopic surgery},
  author={Hwang, Minho and Kwon, Dong-Soo},
  journal={The International Journal of Medical Robotics and Computer Assisted Surgery},
  volume={16},
  number={2},
  pages={e2078},
  year={2020},
  publisher={Wiley Online Library}
}

@article{li2020cadaveric,
  title={Cadaveric feasibility study of a teleoperated parallel continuum robot with variable stiffness for transoral surgery},
  author={Li, Changsheng and Gu, Xiaoyi and Xiao, Xiao and Lim, Chwee Ming and Ren, Hongliang},
  journal={Medical \& Biological Engineering \& Computing},
  volume={58},
  number={9},
  pages={2063--2069},
  year={2020},
  publisher={Springer}
}

@article{chen2021shurui,
  title={The SHURUI system: a modular continuum surgical robotic platform for multiport, hybrid-port, and single-port procedures},
  author={Chen, Yuyang and Zhang, Chao and Wu, Zhonghao and Zhao, Jiangran and Yang, Bo and Huang, Jia and Luo, Qingquan and Wang, Linhui and Xu, Kai},
  journal={IEEE/ASME Transactions on Mechatronics},
  volume={27},
  number={5},
  pages={3186--3197},
  year={2021},
  publisher={IEEE}
}

@article{song2021real,
  title={Real-time multi-object magnetic tracking for multi-arm continuum robots},
  author={Song, Shuang and Ge, Han and Wang, Jiaole and Meng, Max Q-H},
  journal={IEEE Transactions on Instrumentation and Measurement},
  volume={70},
  pages={1--9},
  year={2021},
  publisher={IEEE}
}

@article{wu2021closed,
  title={Closed-loop pose control and automated suturing of continuum surgical manipulators with customized wrist markers under stereo vision},
  author={Wu, Baibo and Wang, Longfei and Liu, Xu and Wang, Linhui and Xu, Kai},
  journal={IEEE Robotics and Automation Letters},
  volume={6},
  number={4},
  pages={7137--7144},
  year={2021},
  publisher={IEEE}
}

@inproceedings{yang2020closed,
  title={A closed-loop controller for a continuum surgical manipulator based on a specially designed wrist marker and stereo tracking},
  author={Yang, Haozhe and Wu, Baibo and Liu, Xu and Xu, Kai},
  booktitle={2020 IEEE/ASME International Conference on Advanced Intelligent Mechatronics (AIM)},
  pages={335--340},
  year={2020},
  
}

@inproceedings{xu2009system,
  title={System design of an insertable robotic effector platform for single port access (SPA) surgery},
  author={Xu, Kai and Goldman, Roger E and Ding, Jienan and Allen, Peter K and Fowler, Dennis L and Simaan, Nabil},
  booktitle={2009 IEEE/RSJ International Conference on Intelligent Robots and Systems},
  pages={5546--5552},
  year={2009},
  
}

@inproceedings{reiter2011learning,
  title={A learning algorithm for visual pose estimation of continuum robots},
  author={Reiter, Austin and Goldman, Roger E and Bajo, Andrea and Iliopoulos, Konstantinos and Simaan, Nabil and Allen, Peter K},
  booktitle={2011 IEEE/RSJ International Conference on Intelligent Robots and Systems},
  pages={2390--2396},
  year={2011},
  
}

@article{wang2023vision,
  title={Vision-based markerless tracking for continuum surgical instruments in robot-assisted minimally invasive surgery},
  author={Wang, Longfei and Zhou, Chang and Cao, Yuzhen and Zhao, Ren and Xu, Kai},
  journal={IEEE Robotics and Automation Letters},
  volume={8},
  number={11},
  pages={7202--7209},
  year={2023},
  publisher={IEEE}
}

@article{wang2025closed,
  title={Closed-Loop Cooperative Manipulation of Deformable Tissue via Visual Feedback Using Multiple Continuum Surgical Manipulators},
  author={Wang, Xiang and Kuang, Haomin and Zhu, Chuanxiang and Xu, Kai},
  journal={IEEE Robotics and Automation Letters},
  year={2025},
  volume={10},
  number={5},
  pages={4396-4403},
  publisher={IEEE}
}

\end{document}